\documentclass[accepted]{uai2023} %

\usepackage[american]{babel}

\usepackage{natbib} %
\usepackage{mathtools} %
\usepackage{booktabs} %
\usepackage{tikz} %
\usepackage{graphicx}
\usepackage{float}
\usepackage[caption = false]{subfig}
\usepackage{booktabs} %
\usepackage{bm}
\usepackage{diagbox}
\usepackage{algorithm}
\usepackage[noend]{algorithmic}
\usepackage{thm-restate}
\usepackage[algo2e,ruled,vlined]{algorithm2e}

\usepackage{amssymb,amsthm}
\usepackage{hyperref}

\usepackage{tikz}
\usepackage{url}
\usepackage{setspace}
\usepackage{xcolor}
\usepackage{soul}
\usepackage{longtable}

\usepackage{times}
\usepackage{varwidth}
\usepackage{graphicx}
\usepackage{wrapfig}
\usepackage{caption}

\usepackage{amssymb}
\usepackage{multirow}
\usepackage{bbm}
\usepackage{graphicx}
\usepackage{url}
\usepackage{setspace}
\usepackage{framed}
\usepackage{xcolor}
\usepackage{soul}
\usepackage{longtable}

\usepackage{times}
\usepackage{varwidth}
\usepackage{graphicx}
\usepackage{wrapfig}

\usepackage[utf8]{inputenc} %
\usepackage[T1]{fontenc}    %
\usepackage{url}            %
\usepackage{booktabs}       %
\usepackage{amsfonts}       %
\usepackage{nicefrac}       %

\usepackage{xspace}
\usepackage{color}
\usepackage{mathrsfs}

\usepackage{booktabs}
\usepackage{comment}

\usepackage{multirow}

\newcommand{\Appendix}[1]{the full version for}

\newcommand{\ltwonorm}[1]{\left\| #1 \right\|_2}

\newtheorem{theorem}{Theorem}[section]
\newtheorem{lemma}[theorem]{Lemma}

\newtheorem{remark}{Remark}

\newtheorem{definition}{Definition}

\newcommand{\C}{\mathcal{C}}

\newcommand{\E}{\mathbb{E}}

\newcommand{\G}{\mathcal{G}}

\newcommand{\I}{\mathbf{I}}
\newcommand{\bI}{\mathbb{I}}

\newcommand{\R}{\mathbb{R}}

\newcommand{\Z}{\mathbb{Z}}

\renewcommand{\comment}[1]{}
\newcommand{\red}[1]{}%

\newcommand{\cA}{\mathcal{A}}
\newcommand{\cB}{\mathcal{B}}

\newcommand{\cD}{\mathcal{D}}

\newcommand{\cH}{\mathcal{H}}

\newcommand{\cL}{\mathcal{L}}

\newcommand{\cP}{\mathcal{P}}

\newcommand{\cY}{\mathcal{Y}}
\newcommand{\cX}{\mathcal{X}}

\newcommand{\bbE}{\mathbb{E}}

\newcommand{\vol}{\mathsf{vol}}

\definecolor{colorY}{rgb}{0.7 , 0.7 , 0.2}

\newenvironment{proofoutline}{\noindent{\emph{Proof Sketch. }}}{\hfill$\square$\medskip}

\title{Efficiently Learning the Graph for Semi-supervised Learning}

\author[1]{\href{mailto:<dravyans@cs.cmu.edu>?Subject=Your UAI 2023 paper}{Dravyansh Sharma}{}}
\author[1]{\href{mailto:<mjones2@andrew.cmu.edu>?Subject=Your UAI 2023 paper}{Maxwell Jones}{}}
\affil[1]{%
    School of Computer Science.\\
    Carnegie Mellon University\\
    Pittsburgh, PA, 15213
}
  
\begin{document}
\onecolumn %
\maketitle

\begin{abstract}
  Computational efficiency is a major bottleneck in using classic graph-based approaches for semi-supervised learning on datasets with a large number of unlabeled examples. Known techniques to improve efficiency typically involve an approximation of the graph regularization objective, but suffer two major drawbacks – first the graph is assumed to be known or constructed with heuristic hyperparameter values, second they do not provide a principled approximation guarantee for learning over the full unlabeled dataset. Building on recent work on learning graphs for semi-supervised learning from multiple datasets for problems from the same domain, and leveraging techniques for fast approximations for solving linear systems in the graph Laplacian matrix, we propose algorithms that overcome both the above limitations. 
  
  We show a formal separation in the learning-theoretic complexity of sparse and dense graph families. We further show how to approximately learn the best graphs from the sparse families efficiently using the conjugate gradient method. 
    Our approach can also be used to learn the graph efficiently online with sub-linear regret, under mild smoothness assumptions. Our online learning results are stated generally, and may be useful for approximate and efficient parameter tuning in other problems. We implement our approach and demonstrate significant ($\sim$10-100x) speedups over prior work on semi-supervised learning with learned graphs on benchmark datasets.
\end{abstract}

\section{Introduction}\label{sec:intro}

As machine learning finds applications in new domains like healthcare, finance and a variety of industrial sectors \citep{vamathevan2019applications,kumar2022kdd,larranaga2018industrial}, obtaining sufficiently large human-annotated datasets for applying supervised learning is often prohibitively expensive. Semi-supervised learning can solve this problem by utilizing unlabeled data, which is more readily available, together with a small amount of human-labeled data. Graph-based techniques, where the similarity of examples is encoded using a graph, are popular and effective for learning using unlabeled data \citep{zhu2009introduction}. Several heuristic approaches for learning given the graph are known, but the choice of a good graph is strongly dependent on the problem domain. How to create the graph has largely been `more of an art than science' \citep{zhu2005semi}, although recent work proposes how to provably learn the best graph for a given problem domain from the data \citep{balcan2021data}. A key limitation of the proposed techniques is their computational efficiency, as the proposed algorithms take $\Tilde{O}(n^4)$ time which make them impractical to run on real datasets. In this work we propose new and more practical approaches that exploit graph {\it sparsity} and employ {\it approximate optimization} to obtain more powerful graph learning techniques with formal guarantees for their effectiveness, and improved efficiency guarantees.

Past work on improving the efficiency of graph-based semi-supervised learning has focused largely on selecting a subset of `important' unlabeled examples. For example, one may use a greedy algorithm \citep{delalleau2005efficient} or a $k$-means based heuristic \citep{wang2016scalable}, run the graph-based algorithm only on the selected subset of examples and use some local interpolation for remaining nodes. In this work we provide more principled approaches that come with formal near-optimality guarantees, and demonstrate the trade-off between accuracy and efficiency. We focus on the data-driven setting, first studied by \citet{balcan2021data} for this problem, where one repeatedly solves multiple semi-supervised learning problems from the same problem domain, and hopes to learn a common graph that works well over the domain. 

We give tools for analysis of regret of online learning algorithms in data-driven algorithm design, applicable beyond semi-supervised graph learning, and  useful in any problem where the loss functions can be  easily approximated.
We also study sample efficiency of the number of problem samples needed to learning a good parameter when the problems come from a distribution over semi-supervised learning problems.
Our work extends prior theoretical results \citep{balcan2021data} on sample efficiency to additional graph families that capture sparsity, obtaining improved sample complexity for sparse graphs. We further propose algorithms which improve over the running time of previous proposed approaches for learning the graph. We employ the conjugate gradient method to compute fast, approximate inverses and optimize over new multi-parameter graph families that include sparse graphs which can be more efficiently optimized.
Empirically, we observe that our proposed approaches are computationally efficient, while retaining the effectiveness guarantees of learning the best graph for the given problem distribution.

In more detail, we learn the graph `bandwidth' hyperparameter for the commonly used Gaussian kernel, optimizing over a continuous parameter domain. This approach is more powerful than a grid search, which only computes results at some finite set of hyperparameter values. We extend the recent line of work on data-driven algorithm design (Section \ref{sec:rw}) to approximate online feedback, and achieve provably near-optimal hyperparameter selection.\looseness-1

\subsection{Main Contributions}\label{sec:contributions}

\begin{itemize}[leftmargin=*,topsep=4pt,partopsep=1ex,parsep=1ex]\itemsep=-4pt
    \item (Section \ref{sec:addag}) We provide a general analysis for online data-driven algorithm design with approximate loss functions, quantifying the accuracy-efficiency trade-off. While prior work on approximate algorithm selection \citep{balcan2020refined} studies bounded $\ell_\infty$-norm approximations in the distributional setting, we generalize along two axes -- we study a more general approximation class necessary to analyse our semi-supervised learning algorithm, and our results apply to online learning, even in the presence of (the more realistic) partial feedback.%
    \item (Section \ref{sec:sparse}) For graph-based semi-supervised learning, we show a formal gap in the pseudodimension of learning sparse and dense graphs. Concretely, if each graph node is connected to at most $K$ neighbors, the pseudodimension is $O(K+\log n)$, which implies an asymptotic gap relative to $\Omega(n)$ bound for learning complete graphs with Gaussian RBF kernels \citep{balcan2021data}. In particular, Theorem \ref{thm:sigma-pdim} implies that our techniques are also more sample efficient, in addition to being more computational efficient.
    \item (Section \ref{sec:algo}) We propose an efficient algorithm based on approximate Laplacian inverse for approximately computing the hyperparameter intervals where the semi-supervised loss objective is constant. We prove  convergence guarantees for our algorithm%
    , which capture a trade-off between the computational efficiency and the accuracy of loss estimation. We instantiate our approach for approximate graph learning for the classic harmonic-objective algorithm of \citet{zhu2003semi}, as well as the computationally efficient algorithm of \citet{delalleau2005efficient}.
    \item (Section \ref{sec:expt}) We implement our algorithm \footnote{\href{https://github.com/maxwelljones14/Efficient-SSL}{https://github.com/maxwelljones14/Efficient-SSL}} and provide extensive empirical study showing improvement over previously proposed approaches on standard datasets. Specifically, we improve the running time by about 1-2 orders of magnitude, while almost retaining (and in some cases slightly increasing) the accuracy.
\end{itemize}

\subsection{Related Work}\label{sec:rw}

\textbf{Approximate Laplacian inverse}. The {\it conjugate gradient method} \citep{hestenes1952methods} is an iterative algorithm used to approximately solve a system $Ax = b$ for symmetric, positive definite matrices. Starting with the zero vector as an approximate solution, every iteration computes a gradient used to update this approximation in the direction of the exact solution. The exact solution itself is obtained in $n$ steps, but good approximate solutions can be found much sooner for graphs with low condition number $\kappa$ \citep{AXELSSON1976123,vishnoi2012laplacian}. Furthermore, each iteration computes a finite number of matrix-vector products on $A$, yielding good runtime guarantees. Many variants of the Conjugate gradient method exist \citep{hager2006survey}, in this work we use the original version. The conjugate gradient method is a tool in use for calculating fast matrix inverses across machine learning applications, in domains such as deep reinforcement learning \citep{schulman2015trust,rajeswaran2017towards}  and market forecasting \citep{SHEN2015243}. We choose the conjugate gradient method over other iterative techniques to solve $Ax = b$ like Lanczos iteration due to its stability, simplicity, and previous success in other machine learning applications. Further, the conjugate gradient method offers strong theoretical guarantees, leading to fast approximate convergence for our use case. %

\textbf{Semi-supervised learning}. {\it Semi-supervised learning} is a paradigm for learning from labeled and unlabeled data \citep{zhu2009introduction,balcan2010discriminative}. 
A popular approach for semi-supervised learning is to optimize a graph-based objective. %
Several methods have been proposed to predict labels {\it given a graph} including $st$-mincuts \citep{blum2001learning}, soft mincuts that optimize a harmonic objective \citep{zhu2003semi}, and label propagation \citep{zhu2002learning}. %
Prior research for efficient semi-supervised learning has also typically assumed that the graph $G$ is given \citep{delalleau2005efficient,wang2016scalable}.
 All algorithms have comparable performance provided the graph $G$ encodes the problem well \citep{zhu2009introduction}. \citet{balcan2021data} introduce a first approach to learn the graph $G$ with formal guarantees, and show that the performance of all the algorithms depends strongly on the graph hyperparameters. In this work, we provide computationally efficient algorithms for learning the graph parameters. While we focus on learning the graph for the  classical approaches, deep learning based approaches also typically assume a graph is available \citep{kipf2017semi}. %

\textbf{Data-driven algorithm design}. \citet{gupta2017pac} define a formal learning framework for selecting algorithms from a family of heuristics or setting hyperparameters. It is further developed by \citet{balcan2017learning,balcan2018dispersion,balcan2020data}. It has been successfully applied to several problems in machine learning like clustering, linear regression and low rank approximation \citep{balcan2017learning,balcan2022provably,bartlett2022generalization} (to list a few) and for giving powerful guarantees like differential privacy, adaptive learning and adversarial robustness \citep{balcan2018dispersion,sharma2020learning,balcan2020power}.  %
\citet{balcan2018dispersion,dick2020semi} introduce general data-driven design techniques under some smoothness assumptions, and \citet{balcan2020refined} study learning with approximate losses. We extend the techniques to broader problem settings (as noted in Section \ref{sec:contributions}, and detailed in Section \ref{sec:addag}), and investigate the structure of graph-based label learning formulation to apply the new techniques. Computational efficiency is an important concern for practical applicability of data-driven design. This includes recent and concurrent work which use fundamentally different techniques like output-sensitive enumeration from computational geometry \citep{balcan2022faster} and discretization for mechanism design applications \citep{balcan2023learning}. In contrast, we study the effectiveness of approximate loss estimation, as well as graph sparsification, in data-driven graph selection for semi-supervised learning.

\section{Notation and Formal Setup}

We are given some labeled points $L$ and unlabeled points $U$. One constructs a graph $G$ by placing (possibly weighted) edges $w(u,v)$ between pairs of data points $u,v$ which are `similar', and labels for the unlabeled examples are obtained by optimizing some graph-based score. We have an oracle $O$ which on querying provides us the labeled and unlabeled examples, and we need to pick $G$ from some family $\G$ of graphs. We commit to using some algorithm $A(G,L,U)$ (or $A_{G,L,U}$) which provides labels for examples in $U$, and we should pick a $G$ such that $A(G,L,U)$ results in small error in its predictions on $U$. To summarize more formally,

{\it Problem statement}: Given data space $\cX$, label space $\cY$ and an oracle $O$ which yields a number of labeled examples  $\emptyset\ne L\subset  \cX\times\cY$ and some unlabeled examples $\emptyset\ne U\subset \cX$ such that $|L|+|U|=n$. We are further given a parameterized family of graph construction procedures over parameter space $\cP$, $\G:\cP\rightarrow(\cX\times\cX\rightarrow \R_{\ge 0})$, graph labeling algorithm $A_{G,L,U}$ which takes a graph $G$ with labeled nodes $L$ and unlabeled nodes $U$ and provides labels for all unlabeled examples in $U$, a loss function $l:\cY\times\cY\rightarrow [0,1]$ and a target labeling $\tau:U\rightarrow \cY$. We need to select $\rho\in \cP$ such that corresponding graph $G(\rho)$ minimizes $\frac{1}{|U|}\sum_Ul(A_{G(\rho),L,U}(u),\tau(u))$ w.r.t. $\rho$.%

We will now describe graph families $\G$ and algorithms $A_{G,L,U}$ considered in this work. We restrict our attention to binary classification, i.e. $\cY=\{0,1\}$, and note that all proposed algorithms  naturally extend to multiclass problems, using the standard one-vs-all trick. We assume there is a feature based {\it similarity function} $d:\cX\times\cX\rightarrow\R_{\ge 0}$, a metric which monotonically captures similarity between the examples. In Definition \ref{defn:g} we formally introduce parametric families to build a graph using the similarity function, which capture and interpolate well-known approaches such as $k$-nearest neighbor graphs, $r$-neighborhood graphs and Gaussian RBF kernels.
In this work, we will consider three parametric families of graph construction algorithms defined below. $\bI[\cdot]$ is the indicator function taking values in $\{0,1\}$. Let $N_k(v)$ denote the set of nodes of $G$ which are the $k$-nearest neighbors of node $v$ under the metric $d(\cdot,\cdot)$. Define $k$-mutual neighborhood as the set of edges for which each end-point is a $k$-nearest neighbor of the other, i.e. $N'_k=\{(u,v)\mid u\in N_k(v) \text{ and } v\in N_k(u)\}$ \citep{ozaki2011using}.\looseness-1

\begin{definition}\label{defn:g} Sparse graph families.\\
a) Thresholded nearest neighbors, $G(k,r)$, (with $k\in\Z^+,r\in\R^+$): $w(u,v)=\bI[d(u,v)\le r \text{ and } (u,v)\in N_k']$ .\\
b) Gaussian nearest neighbors, $G(k,\sigma)$, (with $k\in[K]$ for $K\in\Z^+$, $\sigma\in\R^+$):  $w(u,v)=e^{-\frac{d(u,v)^2}{\sigma^2}}\bI[(u,v)\in N'_k]$.
\end{definition} 
The thresholded nearest neighbor graph adds unweighted edges to $G(k,r)$ only when the examples are closer than some $r\in\R_{\ge 0}$, and are mutual $k$-nearest neighbors. The Gaussian (or RBF) kernel is more powerful and allows weighted edges that depend on the metric distance and the bandwidth parameter $\sigma$. %
We will use $\rho$ to denote a general graph parameter (e.g. $\rho=(k,r)$ for thresholded nearest neighbors) and denote the general parameterized graph family by $G(\rho)$. %
Once the graph is constructed using one of the above families, we can assign labels using a suitable algorithm $A_{G,L,U}$. A popular and effective approach is by optimizing a quadratic objective $\frac{1}{2}\sum_{u,v}w(u,v)(f_u-f_v)^2=f^T(D-W)f$. Here $f$ may either be discrete $f_v\in\{0,1\}$ which corresponds to finding a graph mincut separating the oppositely labeled vertices \citep{blum2001learning}, or $f$ may be continuous, i.e. $f\in[0,1]$, and we can round $f$ to obtain the labels \citep{zhu2003semi}.\looseness-1 %

In the {\it distributional setting}, we are presented with several instances of the graph semi-supervised learning problem assumed to be drawn from an unknown distribution $\cD$ and want to learn the best value of the graph parameter $\rho$. We also assume we get all the labels for the `training' problem instances. A choice of $\rho$ uniquely determines the graph $G(\rho)$ and we use some algorithm $A_{G(\rho),L,U}$ to make predictions (e.g. minimizing the quadratic penalty score above) and suffers loss $l_{A(G(\rho),L,U)}:=\frac{1}{|U|}\sum_Ul(A_{G(\rho),L,U}(u),\tau(u))$ which we seek to minimize relative to smallest possible loss by some graph in the hypothesis space, in expectation over the data distribution $\cD$. We also define the family of loss functions $\mathcal{H}_{\rho}=\{l_{A(G(\rho),L,U)}\mid \rho\in\cP \}$. For example, $\mathcal{H}_{k,r}=\{l_{A(G(k,r),L,U)}\mid (k,r)\in\Z^+\times\R^+ \}$. As a shorthand, we will often denote the loss on a fixed problem instance as a function of the graph hyperparameter $\rho$ as simply $l(\rho)$, and refer to it as the {\it dual semi-supervised loss}.

Finally we note definitions of some useful learning theoretic complexity measures. First recall the definition of pseudodimension \citep{pollard2012convergence} which generalizes VC dimension to real-valued functions, and is a well-known measure for hypothesis-space complexity in statistical learning theory. %
\begin{definition}[Pseudo-dimension] Let $\cH$ be a set of real valued functions from input space $\cX$. We say that
$C = (x_1, \dots, x_m)\in \cX^m$ is pseudo-shattered by $\cH$ if there exists a vector
$r = (r_1, \dots, r_m)\in\R^m$ (called ``witness”) such that for all
$b= (b_1, \dots, b_m)\in \{\pm 1\}^m $ there exists $h_b\in \cH$ such that $\text{sign}(h_b(x_i)-r_i)=b_i$. Pseudo-dimension of $\cH$  is the cardinality of the largest set
pseudo-shattered by $\cH$.
\end{definition}

We will also need the definition of {\it dispersion} \citep{balcan2018dispersion} which, informally speaking, captures how amenable a non-Lipschitz function is to online learning. As noted by \citet{balcan2018dispersion,sharma2020learning}, dispersion is a sufficient condition for learning piecewise Lipschitz functions online.

\begin{definition}[\label{def:dis}Dispersion]
The sequence of random loss functions $l_1, \dots,l_T$ is $\beta$-{\it dispersed} for the Lipschitz constant $L$ if, for all $T$ and for all $\epsilon\ge T^{-\beta}$, we have that, in expectation, at most
$\Tilde{O}(\epsilon T)$ functions (here $\tilde{O}$ suppresses dependence on quantities beside $\epsilon,T$ and $\beta$, as well as logarithmic terms)
are not $L$-Lipschitz for any pair of points at distance $\epsilon$ in the domain $\C$. That is, for all $T$ and  $\epsilon\ge T^{-\beta}$, $\E\!\!\left[\!
\max\!\!\!\!_{\substack{\rho,\rho'\in\C\\||\rho-\rho'||_2\le\epsilon}}\!\big\lvert
\{ t\!\in\![T] \mid l_t(\rho)\!-\!l_t(\rho')>L||\rho\!-\!\rho'||_2\} \big\rvert \!\right]
\le  \Tilde{O}(\epsilon T)$.
\end{definition}

\section{Approximate %
Data-driven Algorithm Design} \label{sec:addag}

\begin{figure}
\centering

\includegraphics[scale=0.48]{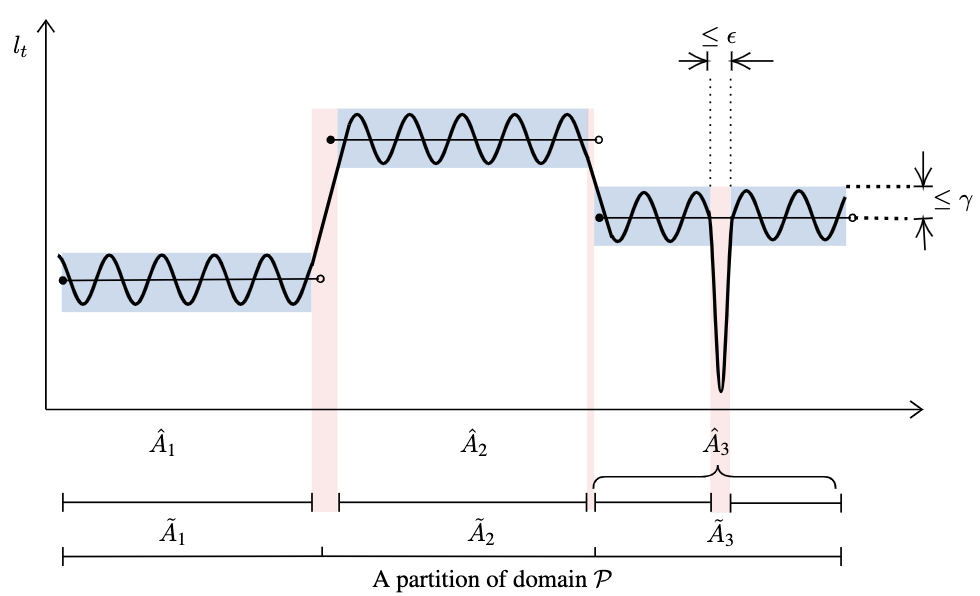}
\caption{A depiction of $(\epsilon,\gamma)$-approximate feedback (Definition \ref{def:fb}) for a one dimensional loss function. Here, the true loss $l_t$ is given by the solid curve, and approximate loss $\Tilde{l}_t$ is piecewise constant.\label{fig:approx-feedback}}

\end{figure}

\citet{balcan2021data} show that the dual loss $l(\rho)$ is a piecewise constant function of the graph hyperparameter $\rho$, for any fixed problem instance. Suppose the problem instances arrive {\it online} in rounds $t=1,\dots,T$, and the learner receives some feedback about her predicted parameter $\rho_t$. A standard performance metric for the online learner is her expected regret, $R_T=\E\left[\sum_{t=1}^Tl_t(\rho_t) -\min_{\rho\in\cP}\sum_{t=1}^Tl_t(\rho)\right]$. Since the {\it full information} setting where all the labels are revealed is not very practical (it assumes all labels of previous problem instances are available), \citet{balcan2021data} also consider the more realistic semi-bandit feedback setting of \citet{dick2020semi}, where only the loss corresponding to hyperparameter $\sigma_t$
selected by the online learner in round $t$ is revealed, along with the end points of the piece $A_t$ containing $\sigma_t$ where $l_t$ is contant. We consider a generalization of this setting where only an approximation to the loss value at $\sigma_t$ is revealed, along with an approximation to the piece $A_t$.\looseness-1

Although our formulation below is motivated by considerations for graph parameter tuning for semi-supervised learning, we provide very general definitions and results that apply to approximate online data-driven parameter selection more generally \citep{balcan2020data}.%

\begin{definition}%
An online optimization problem with loss functions $l_1,l_2,\dots$ is said to have \textbf{$(\epsilon,\gamma)$-approximate semi-bandit feedback} with system size $M$ if for each time $t=1,2,\dots$, there is a partition $\Tilde{A}_t^{(1)},\dots,\Tilde{A}_t^{(M)}$ of the parameter space $\cP\subset\R^d$, called an approximate feedback system, such that if the learner plays point $\rho_t \in \Tilde{A}_t^{(i)}$, she observes the approximate feedback set  $\Tilde{A}_t^{(i)}$%
, and observes approximate loss $\Tilde{l}_t(\rho)$ for all $\rho\in\Tilde{A}_t^{(i)}$ such that $\sup_{\rho\in \hat{A}_t^{(i)}}|\Tilde{l}_t(\rho)-l_t(\rho)|\le \gamma$, for some (unknown) $\hat{A}_t^{(i)}\subseteq \Tilde{A}_t^{(i)}$ with $\Big\lvert\vol\left(\Tilde{A}_t^{(i)}\setminus\hat{A}_t^{(i)}\right)\Big\rvert \le \epsilon$. Here $\vol(A)$ denotes the $d$-dimensional volume of set $A$. We let $\Tilde{A}_t(\rho)$ denote the approximate feedback set that contains $\rho\in\cP$.\looseness-1\label{def:fb}
\end{definition}

For example, let the parameter space $\cP$ be one-dimensional, and in round $t$ the learner plays point $\rho_t\in\cP$. Now suppose the approximate loss functions are also piecewise constant with pieces $\Tilde{A}_t^{(1)},\dots,\Tilde{A}_t^{(M)}$ that partition $\cP$, and she receives information about the constant piece $\Tilde{A}_t(\rho_t)$ containing the played point  by receiving the ends points of interval $\Tilde{A}_t$ %
and approximate loss value $\Tilde{l}_t$ for the observed piece $\Tilde{A}_t$ with $|\Tilde{l}_t-l_t|\le \gamma$ for most of the interval $\Tilde{A}_t$, except possibly finitely many small intervals with total length $\epsilon$, where $l_t$ is the true loss function. This satisfies the definition of $(\epsilon,\gamma)$-approximate semi-bandit feedback. See Figure \ref{fig:approx-feedback} for an illustration. This simple example captures the semi-supervised loss $l_{A(G(\rho),L,U)}$ (where in fact the true loss function is also piecewise constant \citep{balcan2021data}), but our analysis in this section applies to more general {\it piecewise-Lipschitz} loss functions, and for high dimensional Euclidean action space. This approximate feedback model generalizes the ``exact'' semibandit feedback model of \citet{dick2020semi} (which in turn generalizes the standard `full information' setting that corresponds to $M=1$) and is useful for cases where computing the exact feedback set or loss function is infeasible or computationally expensive. Our model also generalizes the approximate loss functions of \citet{balcan2020refined} where positive results (data-dependent generalization guarantees) are shown for $(0,\gamma)$-approximate full-information ($M=1$) feedback in the distributional setting. This extension is crucial for applying our techniques of efficient graph learning by computing approximate loss values for the learned graph.

\begin{algorithm}
\caption{$\textsc{Approximate Continuous Exp3-Set}(\lambda)$}
\label{algorithm: semibandit}
\begin{algorithmic}[1]
\STATE {\bfseries Input:} step size $\lambda\in[0,1]$.
\STATE {Initialize $w_1(\rho)=1$ for all $\rho\in\cP$}.
\FOR{$t=1,\dots,T$}
\STATE{Sample $\rho_t$ according to $p_t(\rho)=\frac{w_t(\rho)}{\int_{\cP}w_t(\rho)d\rho}$.}

\STATE{Play $\rho_t$ and suffer loss $l_t(\rho_t)$.}
\STATE{Observe $(\gamma,\epsilon)$-approximate feedback $\Tilde{l}_t(\rho)$ over set $\Tilde{A}_t$ with $\rho_t\in \Tilde{A}_t$}
\STATE{Update $w_{t+1}(\rho)=w_t(\rho)\exp(-\lambda \hat{l}_t(\rho))$, where $\hat{l}_t(\rho)=\frac{\I\{\rho\in\Tilde{A}_t\}}{\int_{\Tilde{A}_t}p_t(\rho)d\rho}\Tilde{l}_t(\rho)$.}
\ENDFOR
\end{algorithmic}
\end{algorithm}

We give a general online learning algorithm in the presence of approximate semi-bandit feedback, and we show that our algorithm achieves sub-linear regret bounds. In particular, our results indicate how the approximation in the loss function impacts the regret of our algorithm and provides a way to quantify the accuracy-efficiency trade-off (better loss approximation can improve regret in Theorem \ref{thm:approx-feedback}, but at the cost of efficiency in Theorems \ref{thm:harmbound}, \ref{thm:delbound}).

\begin{theorem}\label{thm:approx-feedback}
Suppose $l_1,\dots,l_T:\cP\rightarrow[0,1]$ is a sequence of $\beta$-dispersed loss functions, and the domain  $\cP\subset\R^d$ is contained in a ball of radius $R$. The Approximate Continuous Exp3-Set algorithm (Algorithm \ref{algorithm: semibandit}) achieves expected regret $\tilde{O}(\sqrt{dMT\log(RT)}+T^{1-\min\{\beta,\beta'\}})$ with access to $(\epsilon,\gamma)$-approximate semi-bandit feedback with system size $M$, provided $\gamma \le T^{-\beta'},\epsilon\le {\vol(\cB(T^{-\beta}))} T^{-\beta'}$, where $\cB(r)$ is a $d$-ball of radius $r$.
\end{theorem}

\begin{proofoutline}
    We adapt the $\textsc{Continuous-Exp3-SET}$ analysis of \citet{alon2017nonstochastic,dick2020semi}. %
Define weights $w_t(\rho)$ over the parameter space $\cP$ as $w_{1}(\rho)=1$ and $w_{t+1}(\rho)=w_{t}(\rho)\exp(-\eta\hat{l}_t(\rho))$ and normalized weights $W_t=\int_{\cP}w_t(\rho)d\rho$. Note that $p_t(\rho)=\frac{w_t(\rho)}{W_t}$. We give upper and lower bounds on  the quantity $\E[\log W_{T+1}/W_{1}]$, i.e. the expected value of the log-ratio of normalized weights, and bound the slackness induced in these bounds due to $(\epsilon,\gamma)$-approximate feedback. Our analysis shows that, provided the error terms $\epsilon,\gamma$ are sub-constant in $T$ as stated, we achieve sublinear expected regret. \red{forward reference the implication on runtime complexity}

In more detail, we obtain the following upper bound on $\E[\log W_{T+1}/W_{1}]$, which quantifies the effect of $(\epsilon,\gamma)$-approximation in the semi-bandit feedback.

$$\E\left[\log \frac{W_{T+1}}{W_{1}}\right]\le -\eta\E\left[\sum_{t=1}^Tl_t(\rho_t)\right] + \eta T(M\epsilon+\gamma)+ \frac{\eta^2MT}{2}.$$

We further obtain a lower bound on $\E[\log W_{T+1}/W_{1}]$, again carefully accounting for the approximate feedback.

\begin{align*}\E\left[\log \frac{W_{T+1}}{W_1}\right]
\ge &d\log \frac{r}{R}-\frac{\eta}{\vol(\cB(\rho^*,r))} \sum_{t=1}^T\int_{\cB(\rho^*,r)}{l}_t(\rho)d\rho-\eta\gamma - \frac{\eta M\epsilon}{\vol(\cB(\rho^*,r))}.\end{align*}

\noindent Here $\rho^*\in\cP$ is a minimizer of $\sum_{i=1}^Tl_t$. By the assumption on the number of $L$-Lipschitz violations from Definition \ref{def:dis}, we get $\sum_tl_t(\rho)\ge \sum_tl_t(\rho^*)-TLr-D_r$, or

\begin{align*}\E\left[\log\frac{W_{T+1}}{W_1}\right] \ge &d\log\frac{r}{R}-\eta\sum_{t=1}^Tl_t(\rho^*)-\eta TLr- \eta D_r-\eta\gamma T-\frac{\eta M\epsilon T}{\vol(\cB(\rho^*,r))}.\end{align*}

\noindent Combining the lower and upper bounds gives

\begin{align*}\E\left[\sum_{t=1}^Tl_t(\rho_t)\right] &-\sum_{t=1}^Tl_t(\rho^*) \le \frac{d}{\eta}\log\frac{R}{r} + \frac{\eta MT}{2} +D_r  +T\left(M\epsilon+2\gamma +Lr+\frac{ M\epsilon }{\vol(\cB(\rho^*,r))}\right). \end{align*}

\noindent Finally, setting $r=T^{-\beta}$, $\eta=\sqrt{\frac{2d\log (RT^\beta)}{TM}}, \gamma\le T^{-\beta'}$ and $\epsilon\le \vol(\cB(r))T^{-\beta'}$, and using that the loss sequence is $\beta$-dispersed, we get the desired regret bound.
\end{proofoutline}

In Theorem \ref{thm:approx-feedback}. $\beta’$ measures the net impact of approximate feedback on the regret of Algorithm \ref{algorithm: semibandit}. In particular, it shows that approximation can affect regret when ($\gamma,\epsilon$ are such that) $\beta’<\beta$ and $\beta’<\frac{1}{2}$. The bound in Theorem \ref{thm:approx-feedback} is good for sufficiently small $\gamma,\epsilon$. However, very small $\gamma,\epsilon$ can come at the expense of speed. In more detail, our results in Section 5 discuss how approximate feedback can be algorithmically implemented and useful to obtain faster runtime (runtime bounds are weaker for smaller $\epsilon$). Together, the results quantify an accuracy-efficiency trade-off, and indicate how to set the approximation parameters to improve the efficiency (of graph hyperparameter tuning) without sacrificing the accuracy.

\section{Learning Sparse Graph Families}\label{sec:sparse}

Using (weighted) edges for the $k$-nearest neighbors to use a sparse graph is well-known as an optimization for computational efficiency in semi-supervised learning \citep{delalleau2005efficient,wang2016scalable}. Here we will show that it also formally reduces the learning theoretic complexity, for the problem of graph hyperparameter tuning. %

We can upper bound the pseudodimension of the class of loss functions for sparse graph families, where only $k$ nearest neighbors are connected, for tunable parameter $k\le K$. This upper bound improves on the $O(n)$ bound from \citet{balcan2021data} since $K\le n$, and involves a more careful argument to bound the number of possible label patterns.\looseness-1

\begin{theorem}\label{thm:sigma-pdim}
The pseudo-dimension of $\mathcal{H}_{k,\sigma}$ is $O(K+\log n)$ when the labeling algorithm $A$ is the mincut approach of \citet{blum2001learning}.
\end{theorem}

\begin{proof}
Consider an arbitrary node $u$ in any fixed problem instance. Also fix $k\in [K]$. Since $f(d)=\exp(-d^2/\sigma^2)$ is monotonic in $d$ for any $\sigma>0$, the set $N_k(u)$ of $k$ nearest neighbors of $u$ is the same for all values $\sigma$. This is true for any $u$,  therefore $N_k$ and also the set of mutual nearest neighbors $N'_k(u)=\{v\in N_k(u)\mid (u,v)\in N_k\}$ is also fixed given the pairwise distances for the instance.

We can show that the label of $u$ can flip for at most $O(K2^{2K})$ distinct values of $\sigma$ for the given instance. Suppose that the label of $u$ flips for $\sigma=\sigma_0$ (as $\sigma$ is increased from 0 to infinity), say from positive to negative (WLOG). Let $S_+,S_+'\subseteq N'_k$ for $G(k,\sigma_0^-)$ and $G(k,\sigma_0^+)$ respectively denote the positively labeled neighbors of $u$ just before and after $\sigma=\sigma_0$. Note that $\sigma_0$ is the root of an exponential equation in at most $2k$ terms and therefore has at most $2k$ possible values (Lemma 26 in \citet{balcan2021data}) obtained by comparing the total weights of edges in $\delta(u,N_k'\setminus S_+)$  and $\delta(u,S_+')$, where $\delta(v,V)$ denotes the set of edges with one end-point $v$ and the other end point in vertex set $V$. Over all possible pairs of $S_+,S_+'$ we have at most $2k{2^k\choose 2}=O(K2^{2K})$ possibilities for $\sigma_0$.

The above bound holds for any fixed $k$. For all $k\in[K]$ there are at most $O(K^22^{2K})$ label flips for any fixed node $u$ (as $\sigma$ is varied). Summing up over all $n$ possible choices of $u$ and over all $m$ problem instances, we have at most $O(mnK^22^K)$ intervals of $\sigma$ such that the labelings of all nodes are identical for all instances, for all values of $k$, within a fixed interval. Using Lemma 2.3 of \citet{balcan2020data} (proof of which involves arguments similar to those used in the proof of Theorem \ref{thm:threshold}), the pseudo-dimension $m$ satisfies $2^m\le O(mnK^22^K)$, or $m=O(K+\log n)$.
\end{proof}

The above argument gives a better sample complexity than dense graphs, for which the pseudo-dimension is known to be $\Theta(n)$ \citep{balcan2021data}. We can also give upper bounds on the pseudo-dimension for $H_{k,r}$, the $k$-nearest neighbor graph that adds edges only in $r$-neighborhood, which implies existence of sample and computationally efficient algorithms for learning the best graph parameter $\rho=(k,r)$ using standard results.

\begin{figure*}
\centering
\subfloat[Gradient Descent Only]{\includegraphics[width = 2.2in]{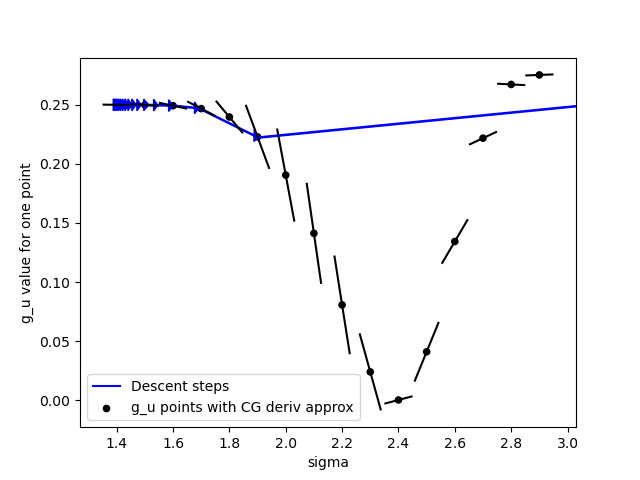}} 
\subfloat[Newton's Method Only]{\includegraphics[width = 2.2in]{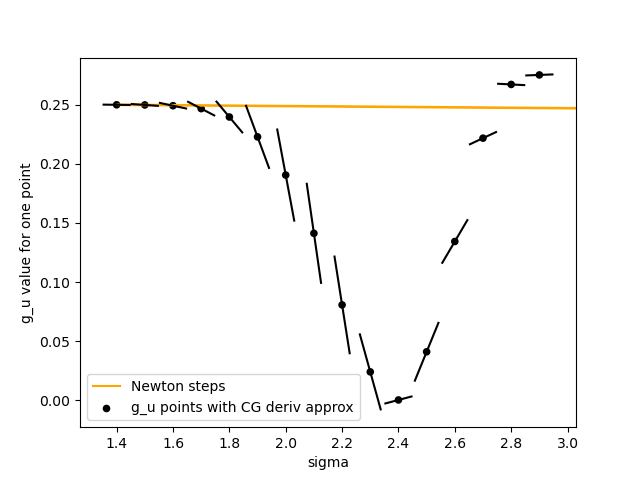}}
\subfloat[Our Method]{\includegraphics[width = 2.2in]{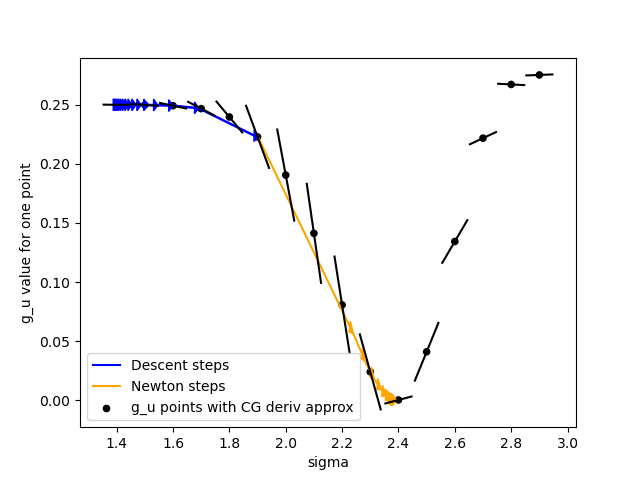}} 
\caption{An instance of a node $u$ for graph $G$ on a subset of the MNIST dataset, where finding local minima of $g_u(\sigma)=(f_u(\sigma)-\frac{1}{2})^2$ is challenging for both Gradient descent and Newton steps. Our method (taking the minimum of the two) finds the local minima}
\label{fig: gd+ns}
\end{figure*}

\begin{theorem}\label{thm:threshold}
The pseudo-dimension of $\mathcal{H}_{k,r}$ is $O(\log n)$ for any labeling algorithm $A$.
\end{theorem}

\begin{proof}
Consider any fixed problem instance with $n$ examples. For any fixed choice of parameter $k$, there are at most $\frac{nk}{2}$ (unweighted) edges in $G(k,r)$ for any value of $r$. Therefore, as $r$ is increased from 0 to infinity, the graph changes only when $r$ corresponds to one of $\frac{nk}{2}$ distinct distances between pairs of data points, and so at most $\frac{nk}{2}+1$ distinct graphs may be obtained for any $k$. Summing over all possible values of $k\in[n]$, we have at most $O(n^3)$ distinct graphs. \looseness-1

\noindent Thus given set $\mathcal{S}$ of $m$ instances $(d^{(i)},L^{(i)},U^{(i)})$, we can partition the real line into $O(mn^3)$ intervals such that all values of $r$ behave identically for all instances, and for all values of $k$, within any fixed interval. Since $A$ and therefore its loss is deterministic once the graph $G$ is fixed, the loss function is identical in each interval. Each piece can have a {\it witness} above or below it as $r$ is varied for the corresponding interval, and so the binary labeling of $\mathcal{S}$ is fixed in that interval. The pseudo-dimension $m$ satisfies $2^m\le O(mn^3)$ and is therefore $O(\log n)$.
\end{proof}

Note that the lower bounds from \citet{balcan2021data} imply that the above bound is asymptotically tight. %
Our bounds in this section imply upper bounds on number of problem instances needed for learning the best parameter values for the respective graph families \citep{balcan2020data} in the distributional setting. More precisely, we can bound the sample complexity of $(\epsilon,\delta)$-uniformly learning (Appendix \ref{app:pac}).\looseness-1

\begin{theorem}[\cite{anthony1999neural}] \label{thm:pdim-generalization}
 Suppose $\cH$ is a class of  functions $\cX\rightarrow[0,1]$ having pseudo-dimension $\textsc{PDim}(\cH)$. For every $\epsilon>0,\delta\in (0,1)$, the sample complexity of $(\epsilon,\delta)$-uniformly learning the class $\cH$ is $O\left(\frac{1}{\epsilon^2}\left(\textsc{PDim}(\cH)\ln\frac{1}{\epsilon}+\ln\frac{1}{\delta}\right)\right)$.
\end{theorem}

\section{Scalability with Approximation Guarantees}\label{sec:algo}

We will now present and analyze an algorithm (Algorithm \ref{algorithm: semi harmonic}) for computing approximate semi-bandit feedback for the dual semi-supervised loss $l(\sigma)$ over  $\sigma\in[\sigma_{\min},\sigma_{\max}]$ (we assume number of nearest-neighbors $k$ is a fixed constant in the following), where $\sigma$ is the Gaussian bandwidth parameter (Def. \ref{defn:g}). Our algorithm is a scalable version of Algorithm 4 of \citet{balcan2021data}. %
Our proposed approach involves two main modifications noted below. %

\begin{itemize}[leftmargin=*,topsep=4pt,partopsep=1ex,parsep=1ex]\itemsep=-4pt
    \item Our Algorithm \ref{algorithm: semi harmonic} uses approximate soft labels $f(\sigma)_\epsilon$ and gradients $\frac{\partial f}{\partial \sigma}_\epsilon$. %
    We use the {\it conjugate gradient} method to compute these approximations, and provide implementations for the harmonic objective minimization approach of \citet{zhu2003semi}, as well as the efficient algorithm of \citet{delalleau2005efficient} with time complexity bounds (Algorithms \ref{algorithm: harmonic approx} and \ref{algorithm: delalleau approx} resp.\ below). %
    \item We use the approximate gradients to locate points where $f(\sigma^*)=\frac{1}{2}$, corresponding to $\sigma$ value where the predicted label flips. We use these points to find $(\epsilon, \epsilon)$-approximate feedback sets. We propose the use of smaller of Newton's step and gradient descent for better convergence to these points (line \ref{algSemiHarmonic-linemin} in Algorithm \ref{algorithm: semi harmonic}; \cite{balcan2021data} use only Newton's method). We motivate this step by giving an example (from a real dataset) where the gradients are both too small and too large near the minima (Figure \ref{fig: gd+ns}). This makes convergence challenging for both gradient descent and Newton's method, but the combination is effective even in this setting. We also give convergence guarantees and runtime bounds for Algorithm \ref{algorithm: semi harmonic} in the presence of approximate gradients (Theorems \ref{thm:harmboundmain}, \ref{thm:delbound}).
\end{itemize}

We first describe how to instantiate the sub-routine $\cA$ to compute approximate soft labels in Algorithm \ref{algorithm: semi harmonic} (full details in Appendix B for the interested reader).

\phantomsection\label{summary: harmonic approx}
Algorithm \ref{algorithm: harmonic approx} computes the soft label that optimizes the harmonic function objective \citep{zhu2003semi} and gradient for a given value of graph parameter $\sigma$ for a fixed unlabeled node $u$. This is accomplished by running the conjugate gradient for given number of iterations to solve systems corresponding to the harmonic function objective and its gradient.

\begin{algorithm}[t]
\caption{$\textsc{HarmonicApproximation}(G,f_L,u, \sigma,\epsilon)$}
\label{algorithm: harmonic approx}
\begin{algorithmic}[1]
\STATE {\bfseries Input:} Graph $G$ with labeled nodes $f_L$, unlabeled node $u$, query parameter $\sigma$, error tolerance $\epsilon$.
\STATE {\bfseries Output:} approximate soft label $f_{u, \epsilon}$ and approximate gradient $\frac{\partial f_u}{\partial \sigma}_\epsilon$.
\STATE{Let $\text{CG}(A, b, t)$ represent running the conjugate gradient method for $t$ iterations to solve  $Ax = b$.}
\STATE{Let $t_\epsilon$ indicate the number of iterations sufficient for $\epsilon$-approximation of $f_u(\sigma)\frac{\partial f_u}{\partial \sigma}$ (Theorem \ref{thm:harmapproxmain}).}
\STATE {Let $f_{U, \epsilon}(\sigma)=\text{CG}\left((I - P_{UU}), P_{UL}f_L, t_\epsilon\right)$, where $D_{ij}:=\bI[i=j]\sum_{k}W_{ik}, P=D^{-1}W$}.
\STATE{Let $\frac{\partial f}{\partial \sigma}_{\epsilon} \!=  \text{CG}\!\left((I-P_{UU}),\left(\frac{\partial P_{UU}}{\partial \sigma}\!f_{U,\epsilon}+\frac{\partial P_{UL}}{\partial \sigma}\!f_L\right)\!, t_\epsilon\!\right)$}, where
\begin{align*}
    \frac{\partial P_{ij}}{\partial \sigma}&=\frac{\frac{\partial w(i,j)}{\partial \sigma}-P_{ij}\sum_{k\in L+U}\frac{\partial w(i,k)}{\partial \sigma}}{\sum_{k\in L+U}w(i,k)},\\
    \frac{\partial w(i,j)}{\partial \sigma}&=\frac{2w(i,j)d(i,j)^2}{\sigma^3}.
\end{align*}
\RETURN{$f_{u, \epsilon}(\sigma), \frac{\partial f_u}{\partial \sigma}_\epsilon$.}
\end{algorithmic}
\end{algorithm}

\begin{algorithm}[h]
\caption{$\textsc{NonParametricApproximation}(G,f_L, i, \sigma, \epsilon)[\tilde{U},\lambda]$}
\label{algorithm: delalleau approx}
\begin{algorithmic}[1]
\STATE {\bfseries Input:} Graph $G$ with labeled nodes $f_L$ and set of unlabeled nodes $U$, unlabeled node $i\in U$, query parameter $\sigma$, error tolerance $\epsilon$.
\STATE {\bfseries Hyperparameters:} Small subset $\tilde{U} \subset U$ (e.g. chosen by the greedy approach of \cite{delalleau2005efficient}, or via \cite{wang2016scalable}), labeled loss regularization coefficient $\lambda$ (see \cite{delalleau2005efficient}). Here $W$ refers to the weight matrix of $\Tilde{U}$ and labels $L$.
\STATE {\bfseries Output:} approximate soft label $\tilde{f}_{i, \epsilon}$ and approximate gradient $\frac{\partial \tilde{f}_i}{\partial \sigma}_\epsilon$.
\STATE{Let $\text{CG}(A, b, t)$ represent running the conjugate gradient method for $t$ iterations to solve equation $Ax = b$.}
\STATE{Let $t_\epsilon$ indicate the number of iterations sufficient for $\epsilon$-approximation (Theorem \ref{thm:delapproxmain}).}
\STATE {\label{AlgDelapproxLineWeighted} Let $\tilde{f}_{i, \epsilon}(\sigma)=\frac{\sum_{j \in \tilde{U} \cup L}W_{ij}(\sigma) f_j(\sigma)_\epsilon}{\sum_{j \in \tilde{U} \cup L}W_{ij}(\sigma)}$, where}
\begin{align*}
    f(\sigma)_\epsilon & = \text{CG}(A, \lambda \overrightarrow{y}, t_\epsilon), \\
    A & = \lambda \Delta_L + Diag(W \mathbf{1}_n) - W, \\
    \left(\Delta_L\right)_{ij} &= \delta_{ij}\delta_{i \in L}, \\
    \overrightarrow{y} &= (y_1,...,y_l,0,...,0)^\top.
\end{align*}
\STATE{Let $\frac{\partial \tilde{f}_i}{\partial \sigma}_{\epsilon} = \frac{\sum_{j \in \tilde{U} \cup L}\frac{\partial W_{ij}}{\partial \sigma}f_j(\sigma) + \sum_{j \in \tilde{U} \cup L}W_{ij}(\sigma) \frac{\partial f_j}{\partial \sigma}_\epsilon + \tilde{f}_{i, \epsilon}(\sigma)\sum_{j \in \tilde{U} \cup L}\frac{\partial W_{ij}}{\partial \sigma}}{\sum_{j \in \tilde{U} \cup L}W_{ij}}$}, where
\begin{align*}
\frac{\partial f}{\partial \sigma}_\epsilon &= -\text{CG}(A, \frac{\partial A}{\partial \sigma} f, t_\epsilon), \\
\frac{\partial A}{\partial \sigma} &= \text{Diag}\left(\frac{\partial W}{\partial \sigma}\bf{1}_n \right) - \frac{\partial W}{\partial \sigma}, \\
\frac{\partial W_{ij}}{\partial \sigma} &= \frac{2 W_{ij}d_{ij}^2}{\sigma^3}. 
\end{align*}
\RETURN{$\tilde{f}_{i, \epsilon}(\sigma), \frac{\partial \tilde{f}_i}{\partial \sigma}_\epsilon$.}
\end{algorithmic}
\end{algorithm}

\begin{theorem} \label{thm:harmapproxmain}
    \textit{Suppose the function } $f : \mathbb{R} \rightarrow \mathbb{R}$ \textit{is convex and differentiable, and that its gradient is Lipschitz continuous with constant }$L > 0$\textit{, i.e. we have that }$|f'(x) - f'(y)| \leq L |x - y|$\textit{ for any }$x,y$. \textit{Then for some $\sigma \in [\sigma_{\min}, \sigma_{\max}]$, where $\left|\frac{\partial f}{\partial \sigma}\right| < \frac{1}{\epsilon \lambda_{\min}(I - P_{UU})}$ on $[\sigma_{\min}, \sigma_{\max}]$, $\kappa(A)$ is condition number of matrix $A$ and $\lambda_{\min}(A)$ is the minimum eigenvalue of $A$, we can find an $\epsilon$ approximation of $f_{u}(\sigma)\frac{\partial f_u}{\partial \sigma}$} \textit{achieving $\left|f_{u}(\sigma)\frac{\partial f_u}{\partial \sigma} - \left(f_{u}(\sigma)\frac{\partial f_u}{\partial \sigma}\right)_\epsilon\right| < \epsilon$, where $f_{u}(\sigma), \frac{\partial f_u}{\partial \sigma}$ are as described in Algorithm \ref{algorithm: harmonic approx} using } $O\left(\sqrt{\kappa(I - P_{UU})}\log\left(\frac{n}{\epsilon \lambda_{\min}(I - P_{UU})}\right)\right)$ conjugate gradient iterations.
\end{theorem}
\begin{proofoutline}
    Here we show that the quadratic objective introduced in \cite{zhu2003semi} as well as its derivative can be solved efficiently using the conjugate gradient method under standard assumptions, given a well conditioned graph (i.e. reasonable minimum/maximum eigenvalues). We apply \cite{AXELSSON1976123}, which states that an $\epsilon$ approximation of the solution $x$ in equation $Ax = b$ can be found in $$O\left(\sqrt{\kappa(A)} \log\frac{1}{\epsilon}\right)$$ iterations using the CG method. In order to find an $\epsilon$-approximation of the product $f \frac{\partial f}{\partial \sigma}$ given we find approximations of $f$ and $\frac{\partial f}{\partial \sigma}$, we show that $f$ and $\frac{\partial f}{\partial \sigma}$ are bounded as well. \\\\
    First, we note that $I - P_{UU}$ is positive definite, as this is required to apply the CG method to solve equation $(I - P_{UU})x = b$ for some vector $b$ \citep{hestenes1952methods}. Next, show that $f$ is bounded by $\frac{1}{\lambda_{\min}(I - P_{UU})}$ and note $\frac{\partial f}{\partial \sigma}$ is bounded by $\frac{1}{\epsilon \lambda_{\min}(I - P_{UU})}$ by the problem statement. 
    An $\epsilon'$ approximation of the CG method gives us $|f_{\epsilon'} - f| \leq \epsilon' f, |\frac{\partial f}{\partial \sigma}_{\epsilon'} - \frac{\partial f}{\partial \sigma}| \leq \epsilon' \frac{\partial f}{\partial \sigma}$. Setting $\epsilon' = \frac{\epsilon^2 \lambda_{\min}(I - P_{UU})}{3}$, we achieve the desired bound of $\left|f_{u}(\sigma)\frac{\partial f_u}{\partial \sigma} - \left(f_{u}(\sigma)\frac{\partial f_u}{\partial \sigma}\right)_\epsilon\right| < \epsilon$\\\\
    By \cite{AXELSSON1976123}, finding $\epsilon'$ approximations using the CG method on positive definite matrix $G$ be done in $$O\left(\sqrt{\kappa(G)} \log \frac{1}{\epsilon'}\right)$$
iterations. Plugging in our $\epsilon'$ value and matrix $I - P_{UU}$, this takes
 $$O\left(\sqrt{\kappa(I - P_{UU})}\log\left(\frac{1}{\epsilon \lambda_{\min}(I - P_{UU})}\right)\right)$$ iterations of the CG method.\\\\
 This bound on the number of iterations of the CG method needed for our desired $\epsilon$ approximation gives theoretical justification for use of the conjugate gradient method with some fixed number of steps in applicable cases. Specifically, we use a fixed number of CG steps to find good approximate intervals in Algorithm \ref{algorithm: semi harmonic}, which uses $f \frac{\partial f}{\partial \sigma}$ on line \ref{line: CGmult}. Full proof of the theorem can be found in Appendix \ref{thm:harmapprox}
\end{proofoutline}

Algorithm \ref{algorithm: delalleau approx} computes the soft label corresponding to the efficient algorithm of \cite{delalleau2005efficient} and gradient for a given value of graph parameter $\sigma$ for a fixed unlabeled node $i$, by running the conjugate gradient for given number of iterations.

\begin{algorithm}[t]
\caption{$\textsc{ApproxFeedbackSet}(G,f_L,\sigma_0,\epsilon, \eta, \cA)$}
\label{algorithm: semi harmonic}
\begin{algorithmic}[1] 
\STATE {\bfseries Input:} Graph $G$ with unlabeled nodes $U$, labels $f_L$, query parameter $\sigma_0$, error tolerance $\epsilon$, learning rate $\eta$, algorithm $\cA$ to estimate soft labels and derivatives at any $\sigma$ (e.g.  Algorithm \ref{algorithm: harmonic approx}).
\STATE {\bfseries Output:} Estimates for piecewise constant interval containing $\sigma_0$, and function value at $\sigma$.
\STATE{Initialize $\sigma_l=\sigma_h=\sigma_0$.}
\FORALL{$u\in U$}
\STATE{Initialize $n=0, \lambda_0 = 1, y_0=\sigma_0$.}
\WHILE{\label{algSemiHarmonic-linewhile} $|\sigma_{n+1}-\sigma_{n}|\ge \epsilon $}
\STATE{Compute $f_{u, \epsilon}(\sigma), \frac{\partial f_u}{\partial \sigma}_\epsilon$ as  $\cA(G,f_L,u,\sigma_n,\epsilon)$%
}
\STATE{Set $g_u(\sigma_n)=(f_{u, \epsilon}(\sigma_n)-\frac{1}{2})^2$, $g_u'(\sigma_n) = 2\left(f_{u, \epsilon}(\sigma_n)-\frac{1}{2}\right)\left(\frac{\partial f_u}{\partial \sigma}_{\epsilon}\right)$ . \label{line: CGmult}}
\STATE{$\xi_\text{GD} \leftarrow \eta  g_u'(\sigma_n)$; $\xi_\text{Newton} \leftarrow 2  \frac{g_u(\sigma_n)}{g_u'(\sigma_n)} $.}
\STATE{\label{algSemiHarmonic-linemin} $y_{n+1}=\sigma_n-\min\{\xi_\text{GD}, \xi_\text{Newton}\}.$}
\IF{ $\xi_\text{GD}<\xi_\text{Newton}$}
\STATE{\label{algSemiHarmonic-linenestupdate} $\lambda_{n + 1} = \frac{1 + \sqrt{1 + 4 \lambda^2_n}}{2}, \gamma_n = \frac{1 - \lambda_{n}}{\lambda_{n + 1}}, \sigma_{n + 1} = (1 - \gamma_n) y_{n + 1} + \gamma_ny_n$}
\ELSE
\STATE\label{algSemiHarmonic-linenewtupdate} {$\sigma_{n + 1} = y_{n + 1}$}
\ENDIF
\STATE{$n\leftarrow n+1$}
\ENDWHILE
\STATE{\label{algSemiHarmonic-lineupdateint} $\sigma_l = \min \{\sigma_l, \sigma_{n+1}\}
, \sigma_h = \max\{\sigma_h, \sigma_{n+1}\}$.}
\ENDFOR
\RETURN{$[\sigma_l,\sigma_h]$, $f_\epsilon(\sigma_0)$.}
\end{algorithmic}
\end{algorithm}

\begin{theorem} \label{thm:delapproxmain}
    Suppose the function  $f : \mathbb{R} \rightarrow \mathbb{R}$ is convex and differentiable, and that its gradient is Lipschitz continuous with constant $L > 0$, i.e. we have that $|f'(x) - f'(y)| \leq L |x - y|$ for any $x,y$. Then for some $\sigma \in [\sigma_{\min}, \sigma_{\max}]$,
    where $\left|\frac{\partial f}{\partial \sigma}\right| \in O\left(\frac{1}{\epsilon \lambda_{\min}(A)}\right)$ on $[\sigma_{\min}, \sigma_{\max}]$, $\kappa(A)$ is 
    condition number of matrix $A$ and $\lambda_{\min}(A)$ is the minimum eigenvalue of $A$, we can find an $\epsilon$ 
    approximation of $\tilde{f}_{u}(\sigma)\cdot\frac{\partial \tilde{f}_u}{\partial \sigma}$ achieving 
    $\left|\tilde{f}_{u}(\sigma)\frac{\partial \tilde{f}_u}{\partial \sigma} - \left(\tilde{f}_{u}(\sigma)\frac{\partial \tilde{f}_u}{\partial \sigma}\right)_\epsilon\right| < \epsilon$, where $\tilde{f}_{u}(\sigma), \frac{\partial \tilde{f}_u}{\partial \sigma}$ are as described in Algorithm \ref{algorithm: delalleau approx} using \\
    $O\left(\sqrt{\kappa(A)}\log\left(\frac{\lambda (|L_{\text{Labels}}| + |\tilde{U}_{\text{Labels}}|)}{\epsilon \sigma_{\min}\lambda_{\min}(A)}\right)\right)$  conjugate gradient iterations. Here $L_{\text{Labels}}$ and $\tilde{U}_{\text{Labels}}$ are sets of labels as described in Algorithm \ref{algorithm: delalleau approx}, and $\lambda$ is the parameter passed into Algorithm \ref{algorithm: delalleau approx}.
\end{theorem}
\begin{proofoutline}
    Algorithm \ref{algorithm: delalleau approx} computes the soft label and gradient corresponding to the efficient algorithm of \citet{delalleau2005efficient} for a graph $G$ with parameter $\sigma$ for a fixed unlabeled node $i \in U$. Unlike Algorithm \ref{algorithm: harmonic approx}, \citet{delalleau2005efficient} only performs a matrix inversion for a small subset of unlabeled 'training' nodes $\tilde{U} \subset U$ along with a set of labeled nodes $L$. The psuedolabels $\Tilde{f}_u(\sigma)$ of $i \in U \setminus \tilde{U}$ ('testing' nodes) are determined by a weighted sum of labels $f_x(\sigma)$, for $x \in \tilde{U} \cup L$, weighted by $W_{ij}(x, i)$.\\\\
 Similar to Algorithm \ref{algorithm: harmonic approx}, a derivative of these pseudolabels w.r.t $\sigma$ is calculated, and matrix inversion is estimated via some number of conjugate gradient iterations. In a similar fashion to Theorem \ref{thm:harmapproxmain}, we achieve a bound on the number of conjugate gradient iterations in order to achieve an $\epsilon$ approximation of  $\tilde{f}_{u}(\sigma)\cdot\frac{\partial \tilde{f}_u}{\partial \sigma}$. Notice that the bound $$O\left(\sqrt{\kappa(A)}\log\left(\frac{\lambda (|L_{\text{Labels}}| + |\tilde{U}_{\text{Labels}}|)}{\epsilon \sigma_{\min}\lambda_{\min}(A)}\right)\right)$$ differs from the bound in Theorem \ref{thm:harmapproxmain} in that the numerator of the log term is linear with respect to the small subset $\Tilde{U}$, as opposed to linear with the number of unlabeled nodes $n$. This is due to the size of the inverted matrix, which is now of size $(|L_{\text{Labels}}| + |\tilde{U}_{\text{Labels}}|)$ instead of size $|U| = n$. The full proof is presented in Appendix \ref{thm:delapprox}. 
\end{proofoutline}

Our main result is the following guarantee on the performance of Algorithm \ref{algorithm: semi harmonic}, which captures the approximation-efficiency trade-off for the algorithm. Compared to the $\Tilde{O}(n^4)$ running time of the approach in \citet{balcan2021data}, our algorithm runs in time $\Tilde{O}(n^2)$ for sparse kNN graphs (i.e. $k$-nearest neighbors with small constant $k$). To achieve this speedup, we replace an $O(n^3)$ matrix inverse for a given unlabeled point with a fixed number of Conjugate Gradient iterations taking time $O(|E_G|)$, where $|E_G|$ is the number of edges for graph $G$ corresponding to the matrix being inverted. Combined with our general algorithm for approximate data-driven algorithm design (Theorem \ref{thm:approx-feedback}), we obtain $\Tilde{O}(\sqrt{T})$ expected regret for online graph parameter tuning with approximate semi-bandit feedback, provided we run Algorithm \ref{algorithm: semi harmonic} with $\epsilon\le\frac{1}{\sqrt{T}}$. For our proof, we will assume that the soft label function $f_u(\sigma)$ is convex and smooth (i.e. derivative w.r.t. $\sigma$ is Lipschitz continuous) for estimating the convergence rates. In Section \ref{sec:expt}, we observe that our algorithm works well in practice even when these assumptions on $f$ are not satisfied, and it would be interesting to extend our analysis to weaker assumptions on the soft label.

\begin{figure*}[t]
\centering
\subfloat[MNIST]{\includegraphics[width = 2.2in]{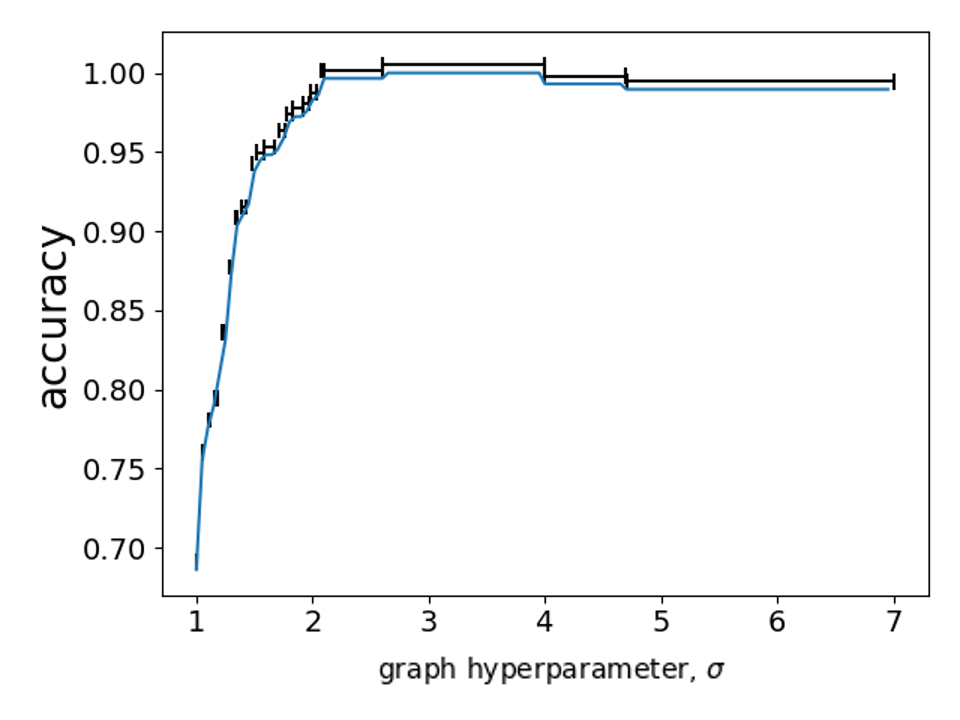}}
\subfloat[Fashion-MNIST]{\includegraphics[width = 2.2in]{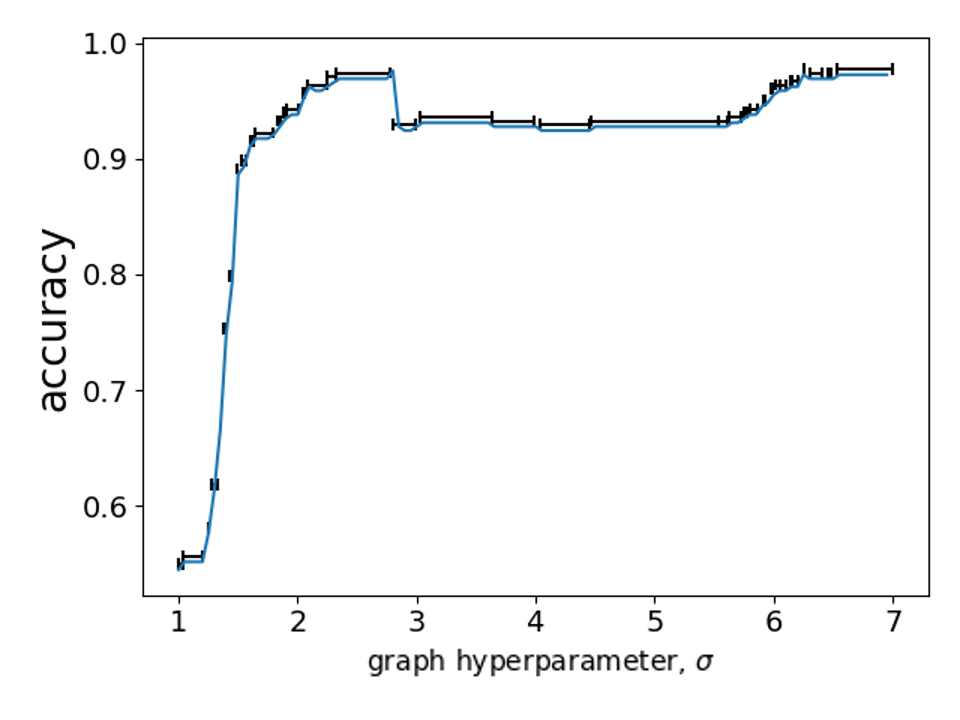}}
\subfloat[USPS]{\includegraphics[width = 2.2in]{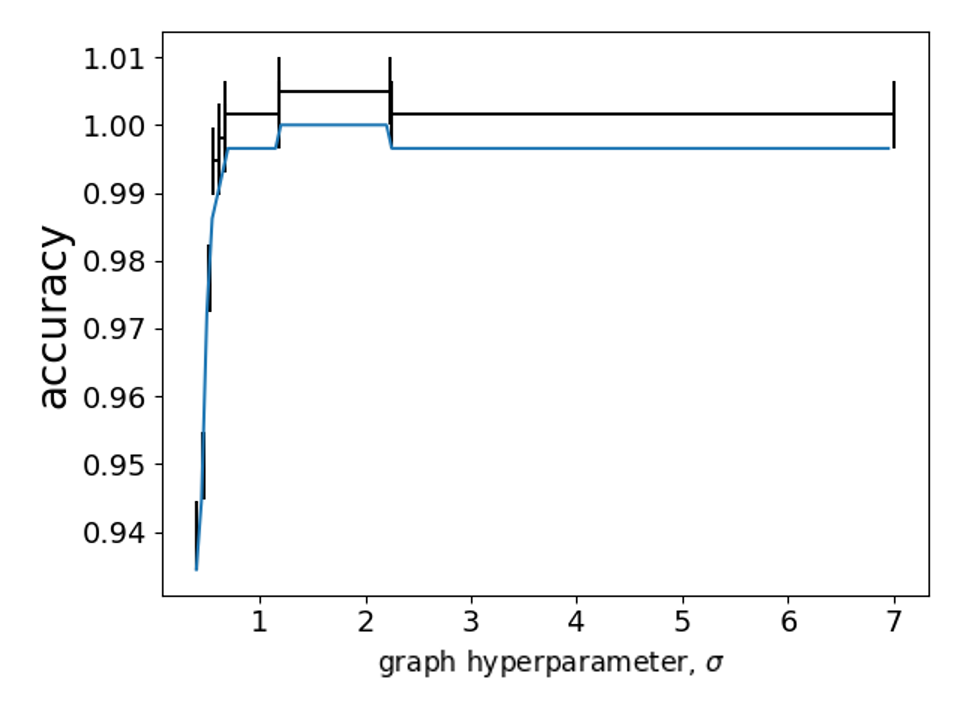}}
\caption{Loss intervals calculated with approximate soft-labels via Algorithm \ref{algorithm: harmonic approx}, kNN = 6, $|U| = 300$. Blue line corresponds to true loss, black intervals are estimated constant loss intervals.}
\label{fig:mainint}
\end{figure*}

\begin{theorem}\label{thm:harmboundmain}
    Using Algorithm \ref{algorithm: harmonic approx} for computing $\epsilon$-approximate soft labels and gradients for the harmonic objective of \citet{zhu2003semi}, if  $f_u(\sigma)$ is convex and smooth, Algorithm \ref{algorithm: semi harmonic} computes $(\epsilon,\epsilon)$-approximate semi-bandit feedback for the semi-supervised loss $l(\sigma)$ %
    in time $O\!\left(\!|E_G|n\sqrt{\kappa(\cL_{UU})}\log\!\left(\!\frac{n\Delta}{\epsilon \lambda_{\min}(\cL_{UU})}\!\right)\log\log\!\frac{1}{\epsilon}\!\right)$, where $|E_G|$ is the number of edges in graph $G$, $\cL_{UU}=I - P_{UU}$ is the normalized grounded graph Laplacian (with labeled nodes grounded), $\Delta=\sigma_{\max} - \sigma_{\min}$ is the size of the parameter range and $\kappa(M)=\frac{\lambda_{\max}(M)}{\lambda_{\min}(M)}$ denotes the condition number of matrix $M$.
\end{theorem}
\begin{proofoutline}
As noted in \citet{balcan2021data}, the loss $l(\sigma)$ is discontinuous at $\sigma^*$ only if $f_{u}(\sigma^*) = \frac{1}{2}$. Algorithm \ref{algorithm: semi harmonic} finds these critical points by finding roots/zeros of $\left(f_u(\sigma) - \frac{1}{2}\right)^2$.  We show (Appendix \ref{thm:nst}) that if $f$ is convex and smooth, Nesterov's accelerated descent \citep{nesterov1983method}  quadratically converges to within $\epsilon$ of such roots, given $O(\frac{\epsilon}{\Delta})$-approximations of $f \frac{\partial f}{\partial \sigma}$ and $\left|\frac{\partial f}{\partial \sigma}\right| < \frac{1}{\epsilon \lambda_{\min}(G_A)}$.   Newton's method   converges quadratically to within $\epsilon$, given $\epsilon$-approximations of $f \frac{\partial f}{\partial \sigma}$ (Appendix \ref{thm:ns}). We use Algorithm \ref{algorithm: harmonic approx} to find suitable $\epsilon$-approximations of $f \frac{\partial f}{\partial \sigma}$ in time $O\left(\sqrt{\kappa(\cL_{UU})}\log\left(\frac{n\Delta}{\epsilon \lambda_{\min}(\cL_{UU})}\right)\right)$ (Theorem \ref{thm:harmapproxmain}). We argue that if the derivative $\frac{\partial f}{\partial \sigma}$ is large (i.e. the condition on $\frac{\partial f}{\partial \sigma}$ for Theorem \ref{thm:nst} does not hold), then the Newton step will be less than $\epsilon$. Since the algorithm uses the smaller of the Newton and Nesterov updates, Algorithm \ref{algorithm: semi harmonic} will terminate for given $u \in U$. By quadratic convergence, we need $O(\log \log \frac{1}{\epsilon})$ iterations in Algorithm \ref{algorithm: semi harmonic} for each of the $O(n)$ unlabeled elements. Finally, noting that the Conjugate Gradient method takes $O(|E_G|)$ time per iteration, we obtain the claimed bound on runtime. For proof details, see Appendix \ref{thm:harmbound}.\looseness-1 %
\end{proofoutline}

Above analysis can be adapted to obtain the following guarantee for tuning $\sigma$ in the efficient algorithm of \citet{delalleau2005efficient}. While the above result guarantees a running time of $\Tilde{O}(n^2)$ for kNN graphs, learning the graph can be done even more efficiently for the scalable approach of \citet{delalleau2005efficient}. Their algorithm minimizes a proxy for the harmonic objective given by

$$\frac{1}{2}\sum_{u,v\in\Tilde{U}}w(u,v)(f(u)-f(v))^2+\lambda\sum_{w\in L}(f(w)-y_w)^2$$

\noindent where $\Tilde{U}\subset U$ and $\lambda$ are hyperparameters. In particular, one chooses a small set $\Tilde{U}$ with $|\Tilde{U}|\ll n$ and efficiently extrapolates the harmonic labels on $\Tilde{U}$ to the rest of $U$ using a Parzen windows based extrapolation. As before, the success of this more efficient approach also depends on the choice of the graph $G$ used. Our Algorithm \ref{algorithm: semi harmonic} obtains good theoretical guarantees in this case as well, with appropriate choice of algorithm $\cA$ (namely Algorithm \ref{algorithm: delalleau approx}).

\begin{figure*}[t]
\centering
\subfloat[MNIST]{\includegraphics[width = 2.2in]{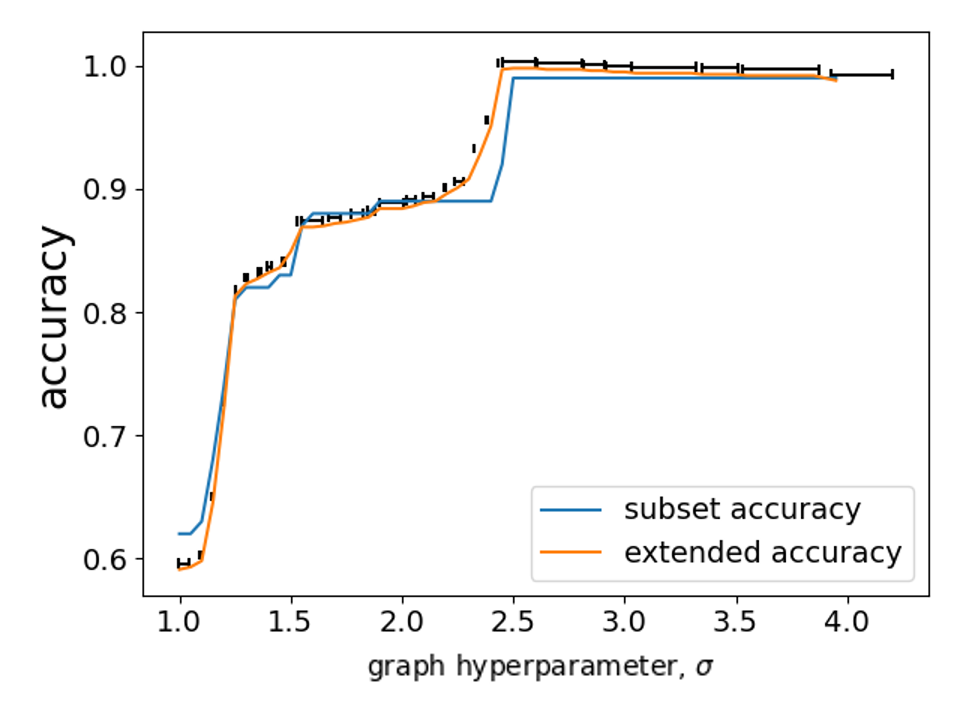}}
\subfloat[Fashion-MNIST]{\includegraphics[width = 2.2in]{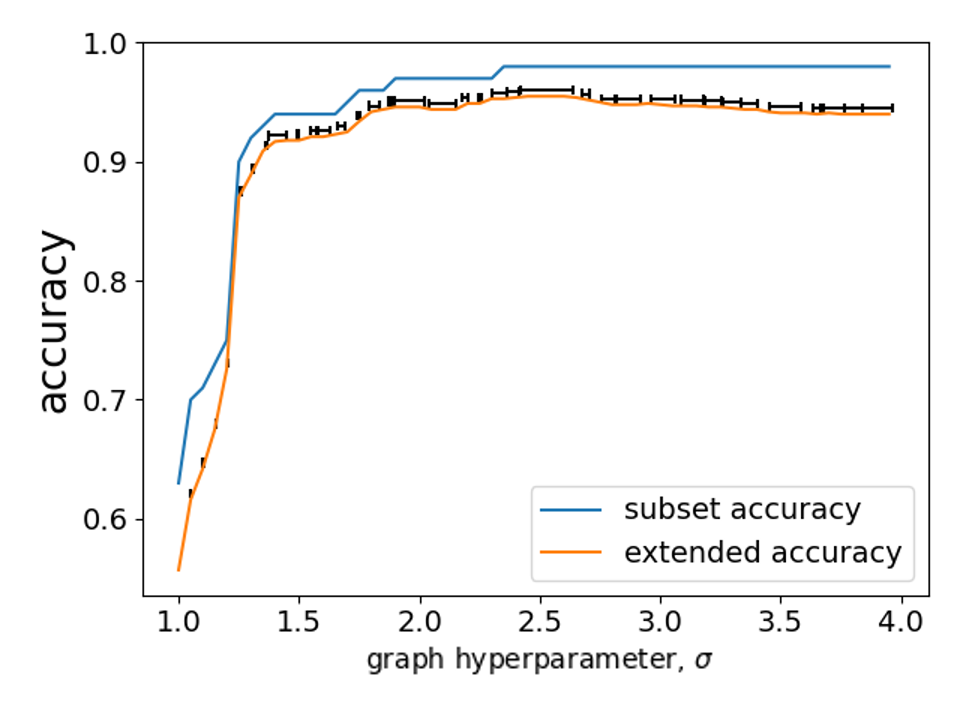}}
\subfloat[USPS]{\includegraphics[width = 2.2in]{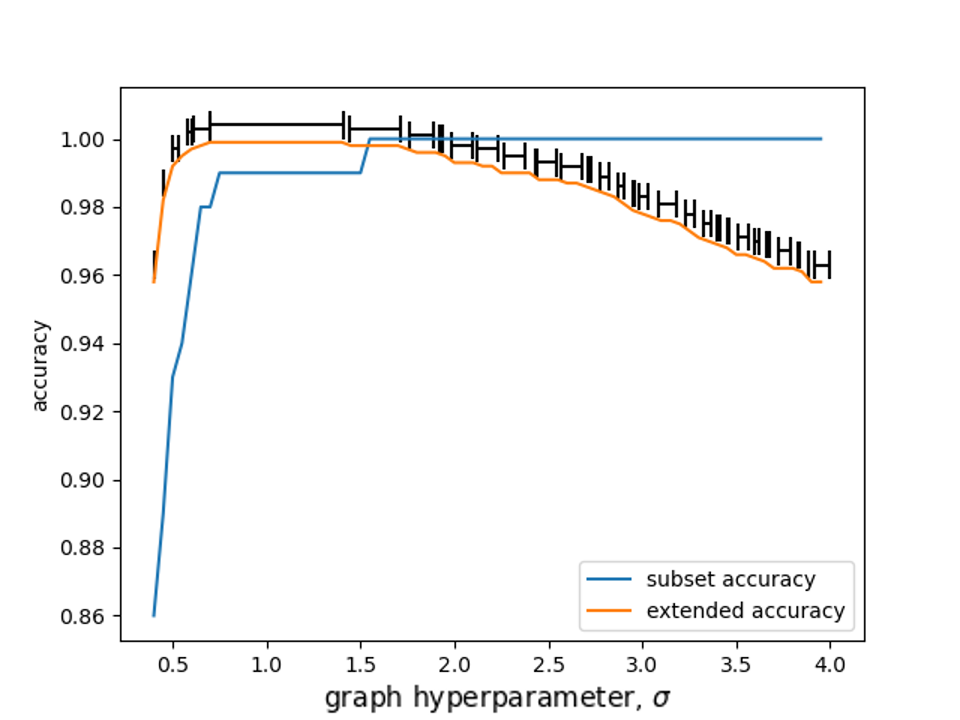}}
\caption{Interval calculation with labeling via Algorithm \ref{algorithm: delalleau approx}. 
$|L| = 10, |U| = 100, |\Tilde{U}| = 1000$.}
\label{fig:del}
\end{figure*}

\begin{theorem}\label{thm:delbound} (Informal)
Given an algorithm for computing $\epsilon$-approximate soft labels and gradients for the efficient semi-supervised learning algorithm of \citet{delalleau2005efficient} (Algorithm \ref{algorithm: delalleau approx}),
Algorithm \ref{algorithm: semi harmonic} computes $(\epsilon,\epsilon)$-approximate semi-bandit feedback for the semi-supervised loss $l(\sigma)$ %
    in time $O\!\!\left(\!|E_{G_{\tilde{U}}}|n\sqrt{\kappa(\cL_{A})}\!\log\!\!\left(\!\frac{\lambda(|L_{\text{Labels}}| + |\tilde{U}_{\text{Labels}}|)\Delta}{\epsilon \sigma_{\min}\lambda_{\min}(\cL_{A})}\!\right)\!\log\log\!\frac{1}{\epsilon}\!\right)$,  where all values are defined as in Algorithm \ref{algorithm: delalleau approx}
\end{theorem}

\begin{proofoutline}
    The proof follows in the same manner as Theorem \ref{thm:harmboundmain}, except we now use Theorem \ref{thm:delapproxmain} to bound  the number of iterations of the CG method. 
\end{proofoutline}

We empirically observe (Figure 5, Appendix E) that sparsity (using only kNN edges, and nodes in $\tilde{U}$) results in well-conditioned matrices (bounded $\sqrt{\kappa(A)}$) in the considered parameter range $[\sigma_{\min},\sigma_{\max}]$.
\begin{remark}
    In this work we have focused on efficient graph learning for the harmonic objective approach of \citet{zhu2003semi} and the efficient algorithm of \citet{delalleau2005efficient}. Developing approaches that work for other efficient algorithms from the literature \citep{sinha2009semi,wang2016scalable} constitutes interesting future work.
\end{remark}

\section{Experiments}\label{sec:expt}

In this section we evaluate computational speedups as a result of using the conjugate gradient method and implement Algorithm \ref{algorithm: semi harmonic} to compute pieces of the loss function under different labeling techniques.

\noindent \textit{Setup}: We consider the task of semi-supervised binary classification (classes 0 and 1) on image datasets. As in \citet{delalleau2005efficient}, we pre-process data instances via Principal Component Analysis, keeping the first 45 principal components. We measure {\it distance} between any pairs of images by $L_2$ distance between principal components, and set weights via Gaussian Kernel parameterized by $\sigma$. %
When testing computational speedup using the CG method, we draw random subsets of the full dataset at varying sizes of $n$, with labeled set size $L = \frac{n}{10}$. When computing approximate matrix inverses, we use $t = 20$ conjugate gradient iterations.

 \noindent\textit{Datasets}:
We use three established benchmark image datasets -- MNIST, Fashion-MNIST, and USPS. Both the MNIST dataset (handwritten digits, \cite{726791}) and the FashionMNIST dataset (mock fashion items, \cite{xiao2017fashion}) consist of 28 by 28 grayscale images with 10 classes, and 6000 images per class. The US Postal Service (USPS) dataset \citep{291440} has 7291 handwritten digits downscaled to 16 by 16 grayscale images. For MNIST and USPS, binary classification between classes 0 and 1 corresonds to classifying between handwritten 0s and 1s. For FashionMNIST, it corresponds to classifying between classes T\_shirt and Trouser.

\subsection{Efficient Feedback Set Computation  (Algorithm \ref{algorithm: semi harmonic})}\label{sec:FeedbackSet}

\begin{table*}[t]
\centering
\caption{Time (in seconds) per Interval (TpI) and Number of Intervals ($M$) using both the Conjugate Gradient method (CG, $t=20$ iterations) and Matrix Inverse (MI) on full graphs or kNN, $k = 6$ graphs using Algorithm \ref{algorithm: semi harmonic}. Approximate soft labels are computed using one of Algorithms \ref{algorithm: harmonic approx}, \ref{algorithm: delalleau approx}.}

\begin{tabular}{ |c|c||c|c|c|c||c|c|c|c||c|c| } 
\hline
\multicolumn{1}{|c|}{\multirow{3}{*}{Dataset}} & \multicolumn{1}{c||}{\multirow{3}{*}{Size}} & \multicolumn{4}{c||}{Algorithm \ref{algorithm: harmonic approx}} & \multicolumn{4}{c||}{Algorithm \ref{algorithm: harmonic approx} (kNN)} &  \multicolumn{2}{c|}{Algorithm \ref{algorithm: delalleau approx} (kNN)} \\ \cline{3-12}

\multicolumn{1}{|c|}{} & \multicolumn{1}{c||}{} & \multicolumn{2}{c|}{TpI} & \multicolumn{2}{c||}{$M$} & \multicolumn{2}{c|}{TpI} & \multicolumn{2}{c||}{$M$} & \multicolumn{1}{c|}{TpI} & \multicolumn{1}{c|}{$M$} \\  \cline{3-12}

\multicolumn{1}{|c|}{} & \multicolumn{1}{c||}{} & \multicolumn{1}{c|}{CG} & \multicolumn{1}{c|}{MI} & \multicolumn{1}{c|}{CG} & \multicolumn{1}{c||}{MI} & \multicolumn{1}{c|}{CG} & \multicolumn{1}{c|}{MI} & \multicolumn{1}{c|}{CG} & \multicolumn{1}{c||}{MI} & \multicolumn{2}{c|}{CG} \\ \hline

\multirow{4}{4em}{MNIST}
& 100 &  \textbf{1.50} & 3.22  & 38.9 & 26.0 &  6.33 & 10.11 & 17.36 & 5.7 &  2.22 & 11.5\\
& 300  &  15.87 & 346.46 & 22.1 &  26.8 &  19.98 & 2405.70 & 22.1 & 7.9 &  \textbf{5.98} & 26.2\\
& 500  &  57.90 & 818.11  & 26.6 & 23.6 & 20.99  & 6791.79 & 26.6 & 7.9 & \textbf{7.60} & 36.1\\
& 12615  &  - & -  & - & - & -  & - & - & - & 46.36 & -\\
\hline
\multirow{4}{4em}{Fashion-MNIST} 
& 100  &  \textbf{1.79} & 3.56 & 39.0 & 23.9 & 11.79 & 12.49 & 18.55 & 8.3 &  2.91 & 12.7\\
& 300   &  11.73 & 268.20 & 45.6 & 38.9 & 21.08 & 1447.56 & 35.9 & 21.2 &  \textbf{7.24} & 33.6\\
& 500  &  39.98 & 766.73 & 50.9 & 37.6 & 35.91 & 6311.03 & 38.7 & 29.0 &  \textbf{9.21} & 44.6\\
& 11950  &  - & -  & - & - & -  & - & - & - & 32.3 & -\\
\hline
\multirow{4}{4em}{USPS} 
& 100  & \textbf{1.58} & 3.47 & 25.4 & 18.6 &  6.91 & 12.53 & 4.7 & {1.3} &  2.12 & 6.67\\
& 300  & 16.63 & 238.16 & 30.1 & 18.8 & 29.86 & 68.70 & 5.6 & {1.2} & \textbf{6.34} & 16.14 \\
& 500 & 57.50 & 755.94 & 39.1 & 17.8 & 63.18 & 31.54 & 6.1 & {1.0} &  \textbf{12.37} & 20.4\\
& 2149  &  - & -  & - & - & -  & - & - & - & 27.28 & -\\
\hline

\end{tabular}   
\label{table:regularIntervals}
\end{table*}

We consider the problem of finding approximate intervals of the piecewise constant loss $l(\sigma)$ using Algorithm \ref{algorithm: harmonic approx} with the number of unlabeled points $n \in \{100, 300, 500\}$. By finding a set of these piecewise constant componenets, we are able to search the continuous paramter space exhaustively, with an optimal hyperparameter being any parameter in the loss interval with lowest loss. We do this for both the complete graph, as well as kNN with $k=6$, setting number of labeled examples $|L|=10$.
We do the same with Algorithm \ref{algorithm: delalleau approx} as algorithm $\cA$, with (uniformly random) subset $\Tilde{U}$ size 50 and hyperparameter $\lambda=1.4$.

We motivate design choices in Algorithm \ref{algorithm: semi harmonic} by examining $(\epsilon,\epsilon)$-approximate semi-bandit feedback for the semi-supervised loss $l(\sigma)$ produced by the algorithm.
For full implementation details, see Appendix E.1.

\noindent \textit{Results: } Figure \ref{fig:mainint} (as well as Fig.\ 2, 3, and 4 in Appendix) indicates that the CG method can be used to find accurate piecewise intervals on real loss functions. We see over 10x speedup for using the CG method as opposed to using matrix inversion for computing soft labels via Algorithm \ref{algorithm: harmonic approx} (Table \ref{table:regularIntervals}). This is due to the speedup in inversion time between a full matrix inverse and the CG method. For specific times for these two inversion techniques, see Tables 1-4 in Appendix. When using Algorithm \ref{algorithm: delalleau approx} to compute labels, we see a speedup of 100x with slower asymptotic increase over time in comparison to the full inverse using Algorithm \ref{algorithm: harmonic approx}. Here the speedup is due to the fact that the CG method is only run for size 50 subsets of $U$, with $O(|\Tilde{U}|)$ algebraic steps afterwards to find soft labels/gradients for a given unlabeled point in $U$. %
We find that the calculation of feedback sets in kNN graphs takes longer to find a single interval on smaller data subsets, but the runtime asymptotically grows slower than for complete graphs. This is likely due to longer piecewise constant intervals for kNN graphs, leading to a larger amount of matrix inverse calculations performed per interval (consistent with $M$ values in Table \ref{table:regularIntervals}). %
We also find that the CG method is more robust to higher condition number than the full inverse for low $\sigma$ values, indicating that the CG method is not only faster, but also can be more stable than full matrix inversion for ill-conditioned grounded Laplacians.\looseness-1 %

\paragraph{CG Method Details.}%
    We run $t$ iterations of the CG method to approximate  grounded Laplacian inverses for finding soft labels in Algorithm \ref{algorithm: harmonic approx}. We consider $t \in \{5,10, 20\}, \sigma' \in [1,7]$ when finding optimal values in terms of unlabeled classification accuracy. %
    We use SciPy \citep{2020SciPy-NMeth} for performing the conjugate gradient method as well as storing nearest-neighbor matrices sparsely for time speedup. For all three datasets, we  find parameters to closely match/surpass the performance of the  harmonic solution with matrix inverse (i.e. prior work) for optimal $\sigma$ in both the complete graph and kNN setting with an order of magnitude or more speedup (Tables 1 and 2 in the appendix). We find that there is little time difference between $t = 5, 15$ and $20$ conjugate gradient iterations. %
    We also find slight speedup %
    for kNN graphs.

\section{Conclusion}

We provide a general analysis for approximate data-driven parameter tuning in the presence of approximate feedback. We show how this approximate feedback may be efficiently implemented for learning the graph in semi-supervised learning using the conjugate gradient method. We further show the significance of using sparse nearest neighborhood graphs for semi-supervised learning -- formally they need provably fewer samples for learning compared to using the complete graph, and moreover, in practice, they lead to better conditioned matrices for which our approach converges faster. We quantify the efficiency vs.\ accuracy trade-off for our approach, and empirically demonstrate its usefulness in more efficiently learning the graph for classic harmonic objective based algorithms \citep{zhu2003semi,delalleau2005efficient}. We believe this is an important step in making the data-driven approach practical for semi-supervised learning, and would potentially be useful for making data-driven algorithm design more useful for other problems. Interesting future directions include extending this approach to learning the graph for modern graph neural network based methods (which still assume a given graph) for semi-supervised learning, and applications to other parameter tuning problems where exact feedback may be computationally expensive to obtain.\looseness-1

\section{Acknowledgement}

We thank Nina Balcan for suggesting the problem and for useful discussions. We also thank David Tang and Advait Nene for useful discussions.

\bibliography{sharma_554}

\clearpage

\appendix

\section{Proof of Theorem \ref{thm:approx-feedback}}\label{app:proof of approx feedback}

\textbf{Theorem 3.1 (restated).} {\it 
Suppose $l_1,\dots,l_T:\cP\rightarrow[0,1]$ is a sequence of $\beta$-dispersed loss functions, and the domain  $\cP\subset\R^d$ is contained in a ball of radius $R$. The Approximate Continuous Exp3-Set algorithm (Algorithm \ref{algorithm: semibandit}) achieves expected regret $\tilde{O}(\sqrt{dMT\log(RT)}+T^{1-\min\{\beta,\beta'\}})$ with access to $(\epsilon,\gamma)$-approximate semi-bandit feedback with system size $M$, provided $\gamma \le T^{-\beta'},\epsilon\le {\vol(\cB(T^{-\beta}))} T^{-\beta'}$, where $\cB(r)$ is a $d$-ball of radius $r$.
}

\begin{proof}[Proof of Theorem \ref{thm:approx-feedback}]
We adapt the $\textsc{Continuous-Exp3-SET}$ analysis of \cite{alon2017nonstochastic,dick2020semi}. %
Define weights $w_t(\rho)$ over the parameter space $\cP$ as $w_{1}(\rho)=1$ and $w_{t+1}(\rho)=w_{t}(\rho)\exp(-\eta\hat{l}_t(\rho))$ and normalized weights $W_t=\int_{\cP}w_t(\rho)d\rho$. Note that $p_t(\rho)=\frac{w_t(\rho)}{W_t}$. We will give upper and lower bounds on  the quantity $\E[\log W_{T+1}/W_{1}]$, i.e. the expected value of the log-ratio of normalized weights.

\noindent Using $\exp(-x) \le 1 - x + x^2/2$ for all $x \ge 0$, we get
\begin{align*}
    \frac{W_{t+1}}{W_t} &= \int_\cP p_t(\rho)\exp(-\eta\hat{l}_t(\rho))d\rho \\&\le 1-\eta \int_\cP p_t(\rho)\hat{l}_t(\rho)d\rho + \frac{\eta^2}{2}\int_\cP p_t(\rho)\hat{l}_t^2(\rho)d\rho.
\end{align*}
Computing the oscillating product and using $1-x\le \exp(-x)$ for all $x\ge 0$, we get
\begin{align*}
    \frac{W_{T+1}}{W_{1}}\le \exp\Bigg(&-\eta \sum_{t=1}^T\int_\cP p_t(\rho)\hat{l}_t(\rho)d\rho + \frac{\eta^2}{2}\sum_{t=1}^T\int_\cP p_t(\rho)\hat{l}_t^2(\rho)d\rho\Bigg).
\end{align*}

\noindent Taking logarithm and expectations on both sides we get

\begin{align*}
    \E\left[\log \frac{W_{T+1}}{W_{1}}\right]\le &-\eta\sum_{t=1}^T\E\left[\int_\cP p_t(\rho)\hat{l}_t(\rho)d\rho\right]+\frac{\eta^2}{2}\sum_{t=1}^T\E\left[\int_\cP p_t(\rho)\hat{l}_t^2(\rho)d\rho\right].
\end{align*}

\noindent %
We have, by the definitions of expectation and approximate semi-bandit feedback,

\begin{align*}\E_t\left[l_t(\rho_t)\right]&=
\int_{\cP}p_t(\rho){l}_t(\rho)d\rho\\
&=
\sum_{i=1}^M \int_{\tilde{A}_t^{(i)}}p_t(\rho){l}_t(\rho)d\rho
\\
&=
\sum_{i=1}^M \left[\int_{\hat{A}_t^{(i)}}p_t(\rho){l}_t(\rho)d\rho+\int_{\tilde{A}_t^{(i)}\setminus \hat{A}_t^{(i)}}p_t(\rho){l}_t(\rho)d\rho\right]\\
&\le \sum_{i=1}^M \int_{\hat{A}_t^{(i)}}p_t(\rho)(\tilde{l}_t(\rho)+\gamma)d\rho+M\epsilon \qquad\qquad\qquad\qquad\qquad(\because\; p_t(\rho){l}_t(\rho)\le 1 \;\forall\;\rho)\\
&\le \sum_{i=1}^M \int_{\tilde{A}_t^{(i)}}p_t(\rho)(\tilde{l}_t(\rho)+\gamma)d\rho+M\epsilon \\
&=\int_{\cP}p_t(\rho)\tilde{l}_t(\rho)d\rho+\gamma  + M\epsilon \qquad\qquad\qquad\qquad\qquad\qquad\quad(\because\; \int_{\cP}p_t(\rho)d\rho=1).%
\end{align*}

Moreover, \begin{align*}\E\left[\int_\cP p_t(\rho)\hat{l}_t(\rho)d\rho\right]&=\E_{<t}\E_t\left[\int_\cP p_t(\rho)\hat{l}_t(\rho)d\rho\right]\\&=\E_{<t}\left[\int_\cP p_t(\rho)\tilde{l}_t(\rho)d\rho\right],\end{align*} using the definition of $\hat{l}_t$ in Algorithm \ref{algorithm: semibandit}. Plugging this in above, we get

\begin{align*}\E\left[l_t(\rho_t)\right]&=\E_{<t}\E_t\left[l_t(\rho_t)\right]\\&\le \E_{<t}\left[\int_\cP p_t(\rho)\tilde{l}_t(\rho)d\rho\right] + \gamma+M\epsilon \\&= \E\left[\int_\cP p_t(\rho)\hat{l}_t(\rho)d\rho\right]+\gamma +M\epsilon,\end{align*}

\noindent and, further,
\begin{align*}
   \E_t[\hat{l}_t(\rho)^2] &=
    \int_\cP p_t(\rho')\left(\frac{\I[\rho\in \Tilde{A}_t(\rho')]}{p_t(\Tilde{A}_t(\rho'))}\tilde{l}_t(\rho)\right)^2d\rho'\\
    &= \left(\frac{\tilde{l}_t(\rho)}{p_t(\Tilde{A}_t(\rho))}\right)^2\int_{\Tilde{A}_t(\rho)}p_t(\rho')d\rho'\\
    &\le\frac{1}{p_t(\Tilde{A}_t(\rho))}.
\end{align*}

\noindent Therefore, $$\E\left[\int_{\cP}p_t(\rho)\hat{l}_t(\rho)^2d\rho\right]\le \E\left[\int_{\cP}p_t(\rho)\cdot \frac{1}{p_t(\Tilde{A}_t(\rho))}d\rho\right]=M.$$ Putting together, we get

$$\E\left[\log \frac{W_{T+1}}{W_{1}}\right]\le -\eta\E\left[\sum_{t=1}^Tl_t(\rho_t)\right] + \eta T(M\epsilon+\gamma)+ \frac{\eta^2MT}{2}.$$

\noindent We can also adapt the argument of \cite{dick2020semi} to obtain a lower bound for $\frac{W_{T+1}}{W_1}$ in terms of $D_r$, the number of $L$-Lipschitz violations between the worst pair of points within distance $r$ across the $T$ loss functions. We have

\begin{align*}\frac{W_{T+1}}{W_1} &=\frac{1}{\vol(\cP)}\int_{\cP}w_{T+1}(\rho)d\rho \\&\ge \frac{1}{\vol(\cP)}\int_{\cB(\rho^*,r)}w_{T+1}(\rho)d\rho. \end{align*}

\noindent Taking the log and applying Jensen's inequality gives

\begin{align*}\log \frac{W_{T+1}}{W_1}\ge& \log \frac{\vol(\cB(\rho^*,r))}{\vol(\cP)}-\frac{\eta}{\vol(\cB(\rho^*,r))} \int_{\cB(\rho^*,r)}\sum_{t=1}^T\hat{l}_t(\rho)d\rho.\end{align*}

\noindent Taking expectations w.r.t. the randomness in Algorithm \ref{algorithm: semibandit} (but for any loss sequence $l_1,\dots,l_t$) and using the fact that $\cP$ is contained in a ball of radius $R$, we get

\begin{align*}\E\left[\log \frac{W_{T+1}}{W_1}\right]
&\ge d\log \frac{r}{R}-\frac{\eta}{\vol(\cB(\rho^*,r))} \sum_{t=1}^T\E\left[\int_{\cB(\rho^*,r)}\hat{l}_t(\rho)d\rho\right].\end{align*}

\noindent Using $\E[\hat{l}_t(\rho)]=\Tilde{l}_t(\rho)$, and noting that for any fixed $t$ and $r$

\begin{align*}
\int_{\cB(\rho^*,r)}\tilde{l}_t(\rho)d\rho&=\sum_{i=1}^M\int_{\cB(\rho^*,r)\cap \Tilde{A}_t^{(i)}}\tilde{l}_t(\rho)d\rho \\&\le \sum_{i=1}^M\int_{\cB(\rho^*,r)\cap \hat{A}_t^{(i)}}\tilde{l}_t(\rho)d\rho+M\epsilon\\
&\le \sum_{i=1}^M\int_{\cB(\rho^*,r)\cap \hat{A}_t^{(i)}}(l_t(\rho)+\gamma)d\rho+M\epsilon\\&\le \sum_{i=1}^M\int_{\cB(\rho^*,r)\cap \tilde{A}_t^{(i)}}\tilde{l}_t(\rho)d\rho+M\epsilon\\&
=\int_{\cB(\rho^*,r)}l_t(\rho)d\rho + \vol(\cB(\rho^*,r))\gamma + M\epsilon,
\end{align*}

\noindent we get that

\begin{align*}\E\left[\log \frac{W_{T+1}}{W_1}\right]
\ge &d\log \frac{r}{R}-\frac{\eta}{\vol(\cB(\rho^*,r))} \sum_{t=1}^T\int_{\cB(\rho^*,r)}{l}_t(\rho)d\rho-\eta\gamma - \frac{\eta M\epsilon}{\vol(\cB(\rho^*,r))}.\end{align*}

\noindent By above assumption on the number of $L$-Lipschitz violations we get $\sum_tl_t(\rho)\ge \sum_tl_t(\rho^*)-TLr-D_r$, or

\begin{align*}\E\left[\log\frac{W_{T+1}}{W_1}\right] \ge &d\log\frac{r}{R}-\eta\sum_{t=1}^Tl_t(\rho^*)-\eta TLr- \eta D_r-\eta\gamma T-\frac{\eta M\epsilon T}{\vol(\cB(\rho^*,r))}.\end{align*}

\noindent Combining the lower and upper bounds gives

\begin{align*}\E\left[\sum_{t=1}^Tl_t(\rho_t)\right] &-\sum_{t=1}^Tl_t(\rho^*) \le \frac{d}{\eta}\log\frac{R}{r} + \frac{\eta MT}{2} +D_r  +T\left(M\epsilon+2\gamma +Lr+\frac{ M\epsilon }{\vol(\cB(\rho^*,r))}\right). \end{align*}

\noindent Finally, setting $r=T^{-\beta}$, $\eta=\sqrt{\frac{2d\log (RT^\beta)}{TM}}, \gamma\le T^{-\beta'}$ and $\epsilon\le \vol(\cB(r))T^{-\beta'}$, and using that the loss sequence is $\beta$-dispersed, we get the desired regret bound

\begin{align*}\E\left[\sum_{t=1}^Tl_t(\rho_t)-\sum_{t=1}^Tl_t(\rho^*)\right]\le O(\sqrt{dMT\log(RT)}+T^{1-\beta}+T^{1-\beta'})  = O(\sqrt{dMT\log(RT)}+T^{1-\min\{\beta,\beta'\}}).\end{align*}

\noindent In particular, we have used $\vol(\cB(T^{-\beta}))\le \vol(\cB(1{}))\le \frac{8\pi^2}{15}$ for any $d$, $T\ge 1$ and $\beta\ge 0$.
\end{proof}

\subsection{Sample complexity for uniform learning.}\label{app:pac}

Let $h^*:\cX\rightarrow\{0,1\}$ denote the target concept. We say $\cH$ is {\it $(\epsilon,\delta)$-uniformly learnable} with sample complexity $n$ if, for every distribution $\cD$, given a sample $S\sim\cD^n$ of size $n$, with probability at least $1 - \delta$, $\big\lvert \frac{1}{n}\sum_{s\in S}|h(s)-h^*(s)| - \bbE_{s\sim\cD}[|h(s)-h^*(s)|] \big\rvert < \epsilon$ for every $h\in\cH$. It is well-known that $(\epsilon,\delta)$-uniform learnability with $n$ samples implies $(\epsilon,\delta)$-PAC learnability with $n$ samples \citep{anthony1999neural}.

\section{Approximate Soft Label and Gradient Computation}\label{app:conjugate gradient}

The piecewise constant interval computation in Algorithm \ref{algorithm: semi harmonic} needs computation of soft labels $f(\sigma)$ as well as gradients $\frac{\partial f}{\partial \sigma}$ for all unlabeled nodes. Typically, one computes a matrix inverse to exactly compute these quantities, and the exact matrix inverted is different for different approaches. In this section, we provide approximate but more efficient procedures for computing these quantities for computing soft labels using the Harmonic objective approach of \cite{zhu2003semi}, as well as for the scalable approach of \cite{delalleau2005efficient}. We also provide convergence guarantees for our algorithms, in terms of the number of conjugate gradient iterations needed for obtaining an $\epsilon$-approximation to the above quantities. Note that replacing $\text{CG}(A, b, t)$ by the computation $A^{-1}b$ recovers the algorithm from \cite{balcan2021data}, which is more precise but takes longer ($O(n^3)$ time or $O(n^\omega)$, where $\omega$ is the matrix multiplication exponent, for the matrix inversion step).

\subsection{Approximate Efficient Soft-labeling of \cite{zhu2003semi}}

\noindent We provide an approximation guarantee for Algorithm \ref{algorithm: harmonic approx}. We first need a simple lemma to upper bound matrix vector products for positive definite matrices.

\begin{lemma} \label{lem:maxeigen} Suppose matrix $A \in \mathbb{R}^{n \times n}$ is positive definite, with $x \in \mathbb{R}^n$. Then $\|Ax\|_2 \leq \lambda_{\text{max}}\|x\|_2$ where $\lambda_{\text{max}}$ is the maximum eigenvalue of $A$\end{lemma}
\begin{proof}
    The idea is to normalize vector $x$, then consider SVD of $A$. Since the vectors are orthonormal, we will be able to simplify to a form that can be upper bounded by $\lambda_{\text{max}}\|x\|_2$. Letting $\hat{x} = \frac{x}{\|x\|}$ and $\{\phi_i\}_{i \in [n]}$ be an orthonormal basis for $A$, we can write $\hat{x}$ as a linear combination of $\{\phi_i\}$: 
    $$\hat{x} = \sum_{i \in [n]} \alpha_i \phi_i.$$ 
    Now,
    \begin{align*}
         \ltwonorm{A\hat{x}}^2
         &= \ltwonorm{ A\sum_{i \in [n]} \alpha_i \phi_i}^2 \\
         &= \ltwonorm{\sum_{i \in [n]} \alpha_i \lambda_i \phi_i}^2 \\
         &= \sum_{i \in [n]} \alpha_i^2 \lambda_i^2 &&(\phi_i \text{ orthonormal}) \\
        &\leq  \lambda_{\text{max}}^2 &&\text{($\lambda_i\le \lambda_{\text{max}}\forall i$; $\hat{x}$ is a unit vector).}
    \end{align*}
    Thus, $\|Ax\| \leq \lambda_{\text{max}} \|x\|$
    using $\hat{x} = \frac{x}{\|x\|}$.
\end{proof}

Equipped with this lemma, we are ready to prove our approximation guarantee for Algorithm \ref{algorithm: harmonic approx}.

\begin{theorem} \label{thm:harmapprox}
    \textit{Suppose the function } $f : \mathbb{R} \rightarrow \mathbb{R}$ \textit{is convex and differentiable, and that its gradient is Lipschitz continuous with constant }$L > 0$\textit{, i.e. we have that }$|f'(x) - f'(y)| \leq L |x - y|$\textit{ for any }$x,y$. \textit{Then for some $\sigma \in [\sigma_{\min}, \sigma_{\max}]$, where $\left|\frac{\partial f}{\partial \sigma}\right| < \frac{1}{\epsilon \lambda_{\min}(I - P_{UU})}$ on $[\sigma_{\min}, \sigma_{\max}]$, $\kappa(A)$ is condition number of matrix $A$ and $\lambda_{\min}(A)$ is the minimum eigenvalue of $A$, we can find an $\epsilon$ approximation of $f_{u}(\sigma)\frac{\partial f_u}{\partial \sigma}$} \textit{achieving $\left|f_{u}(\sigma)\frac{\partial f_u}{\partial \sigma} - \left(f_{u}(\sigma)\frac{\partial f_u}{\partial \sigma}\right)_\epsilon\right| < \epsilon$, where $f_{u}(\sigma), \frac{\partial f_u}{\partial \sigma}$ are as described in Algorithm \ref{algorithm: harmonic approx} using } $O\left(\sqrt{\kappa(I - P_{UU})}\log\left(\frac{n}{\epsilon \lambda_{\min}(I - P_{UU})}\right)\right)$ conjugate gradient iterations.
\end{theorem}
\begin{proof}

A {\it grounded}  Laplacian (aka Dirichlet Laplacian) matrix is obtained by ``grounding'', i.e. removing rows and columns corresponding to, some subset of graph nodes from the Laplacian matrix $L=D-W$.
It is well known that the grounded Laplacian matrix is positive definite \citep{varga1962matrix,miekkala1993graph}. In particular, $\cL_{UU}=D_{UU}-W_{UU}$ and therefore $I-P_{UU}=D^{-1/2}_{UU}\cL_{UU}D^{-1/2}_{UU}$ are positive definite.
This implies $(I - P_{UU})^{-1}$ is also positive definite with maximum eigenvalue $\frac{1}{\lambda_{\min}}$, where $\lambda_{\min}$ is the minimum eigenvalue for $I - P_{UU}$. From here, note that all elements of $ P_{UL}f_L$ are less than one as all labels are 0 or 1, and $P$ is positive in all terms and row normalized to have rowsums of 1. Therefore, %

$$\|f(\sigma)\| = \|(I - P_{UU})^{-1} P_{UL}f_L\|%
\leq \frac{1}{\lambda_{\min}} \|P_{UL}f_L\| \leq \frac{\sqrt{n}}{\lambda_{\min}}$$ 
where the first inequality holds via Lemma \ref{lem:maxeigen}. \\\\
We now have that $\|f(\sigma)\|$ is bounded by $\frac{\sqrt{n}}{\lambda_{\min}(I - P_{UU})}$ 
 on $[\sigma_{\text{min}}, \sigma_{\text{max}}]$. To find an $\epsilon$ approximation in the sense $\|f - f_\epsilon\| \leq  \epsilon$, we set 
$$\epsilon' = \epsilon\lambda_{\min}(I - P_{UU}) \leq \frac{\sqrt{n}\epsilon}{\max_{\sigma \in [\sigma_{\text{min}}, \sigma_{\text{max}}]}f(\sigma)}$$
and note
$$\|f - f_{\epsilon'}\| \leq  \epsilon' \|f\| \leq  \epsilon$$
We consider this process for  $\frac{\partial f}{\partial \sigma}$ as well since $\frac{\partial f}{\partial \sigma}$ is bounded by $\frac{1}{\epsilon\lambda_{\min}(I - P_{UU})}$. Setting $\epsilon' = \epsilon^2 \lambda_{\min}(I - P_{UU})$, $$\left|\frac{\partial f}{\partial \sigma} - \frac{\partial f}{\partial \sigma}_\epsilon\right| \leq \epsilon' \frac{\partial f}{\partial \sigma} \leq \epsilon$$ holds. Finally, letting $$\epsilon' = \frac{\sqrt{n}\epsilon^2 \lambda_{\min}(I - P_{UU})}{3}$$ we achieve the desired result 
$$\left|f \frac{\partial f}{\partial \sigma} - f_{\epsilon'} \frac{\partial f}{\partial \sigma}_{\epsilon'}\right| < \epsilon' f + \epsilon' \frac{\partial f}{\partial \sigma} + \epsilon'^2 < \epsilon.$$
Next we analyze the number of iterations of the CG method used. By \cite{AXELSSON1976123}, finding $\epsilon'$ approximations using the CG method on positive definite matrix $G$ be done in $$O(\sqrt{\kappa(G)}) \log \frac{1}{\epsilon'}$$
iterations. Here we need an $\epsilon' = \frac{\sqrt{n}\epsilon^2 \lambda_{\min}(I - P_{UU})}{3}$ approximation for matrix $I - P_{UU}$, so this takes
 $$O\left(\sqrt{\kappa(I - P_{UU})}\log\left(\frac{n}{\epsilon \lambda_{\min}(I - P_{UU})}\right)\right)$$ iterations of the CG method.
\end{proof}

\subsection{Approximate Efficient Soft-labeling of \cite{delalleau2005efficient}}

\noindent We provide an approximation guarantee for Algorithm \ref{algorithm: delalleau approx}.

\begin{theorem} \label{thm:delapprox}
    Suppose the function  $f : \mathbb{R} \rightarrow \mathbb{R}$ is convex and differentiable, and that its gradient is Lipschitz continuous with constant $L > 0$, i.e. we have that $|f'(x) - f'(y)| \leq L |x - y|$ for any $x,y$. Then for some $\sigma \in [\sigma_{\min}, \sigma_{\max}]$,
    where $\left|\frac{\partial f}{\partial \sigma}\right| \in O\left(\frac{1}{\epsilon \lambda_{\min}(A)}\right)$ on $[\sigma_{\min}, \sigma_{\max}]$, $\kappa(A)$ is 
    condition number of matrix $A$ and $\lambda_{\min}(A)$ is the minimum eigenvalue of $A$, we can find an $\epsilon$ 
    approximation of $\tilde{f}_{u}(\sigma)\cdot\frac{\partial \tilde{f}_u}{\partial \sigma}$ achieving 
    $\left|\tilde{f}_{u}(\sigma)\frac{\partial \tilde{f}_u}{\partial \sigma} - \left(\tilde{f}_{u}(\sigma)\frac{\partial \tilde{f}_u}{\partial \sigma}\right)_\epsilon\right| < \epsilon$, where $\tilde{f}_{u}(\sigma), \frac{\partial \tilde{f}_u}{\partial \sigma}$ are as described in Algorithm \ref{algorithm: delalleau approx} using \\
    $O\left(\sqrt{\kappa(A)}\log\left(\frac{\lambda (|L_{\text{Labels}}| + |\tilde{U}_{\text{Labels}}|)}{\epsilon \sigma_{\min}\lambda_{\min}(A)}\right)\right)$  conjugate gradient iterations. Here $L_{\text{Labels}}$ and $\tilde{U}_{\text{Labels}}$ are sets of labels as described in Algorithm \ref{algorithm: delalleau approx}, and $\lambda$ is the parameter passed into Algorithm \ref{algorithm: delalleau approx}.
\end{theorem}
\begin{proof}
As noted in the proof of \ref{thm:harmapprox}, the grounded Laplacian $A$ is positive definite.
    We can now bound  $A^{-1}$ as in Theorem $\ref{thm:harmapprox}$ and note that 
    $$A^{-1} \lambda \overrightarrow{y} \leq \frac{\lambda \sqrt{|L_\text{Labels}|}}{\lambda_{\min}(A)} $$ via Lemma \ref{lem:maxeigen} as the vector $\overrightarrow{y}$ contains at most $L_{\text{labels}}$  elements with value 1. Note that $\lambda$ is the constant passed in to Algorithm \ref{algorithm: delalleau approx}, and $\lambda_{\min}(A)$ is the smallest eigenvalue of $A$. 
    
    Next, we argue that we can find $\epsilon$ approximations of $f, \frac{\partial f}{\partial \sigma}$ with $\epsilon' = \frac{\sqrt{|L_\text{Labels}|}\epsilon^2\lambda_{\min}(A)}{\lambda}$ similarly to Theorem \ref{algorithm: harmonic approx} as well. 
    From here we consider $\Tilde{f}(\sigma)$ and note that
    \begin{align*}
         \left|\frac{\sum_{j \in \tilde{U} \cup L}W_{ij}(\sigma)f(\sigma)}{\sum_{j \in \tilde{U} \cup L}W_{ij}} - \frac{\sum_{j \in \tilde{U} \cup L}W_{ij}(\sigma)f(\sigma)_\epsilon}{\sum_{j \in \tilde{U} \cup L}W_{ij}}\right| 
        <  \left|\frac{\sum_{j \in \tilde{U} \cup L}W_{ij}}{\sum_{j \in \tilde{U} \cup L}W_{ij}}\epsilon\right| 
        =  \epsilon
    \end{align*}
    Finally, we show that the result holds for $\frac{\partial \tilde{f}_i}{\partial \sigma}$, noting we have proven the result for both $\tilde{f}_{i, \epsilon}$ and $\frac{\partial f}{\partial \sigma}_\epsilon$, and noting that we have exact values for $W_{ij}$ and $\frac{\partial W_{ij}}{\partial \sigma}$
    \begin{align*}
        \left|\frac{\partial \tilde{f}_i}{\partial \sigma}_{\epsilon} - \frac{\partial \tilde{f}_i}{\partial \sigma}\right| 
        &= \frac{\sum_{j \in \tilde{U} \cup L} W_{ij}(\sigma) \epsilon + \epsilon \sum_{j \in \tilde{U} \cup L}{\frac{\partial W_{ij}}{\partial \sigma}}}{\sum_{j \in \tilde{U} \cup L}W_{ij}(\sigma)} \\
        &=  \epsilon + \epsilon \frac{\sum_{j \in \tilde{U} \cup L}\frac{\partial W_{ij}}{\partial \sigma}}{\sum_{j \in \tilde{U} \cup L}W_{ij}(\sigma)} \\
        &= \epsilon + \frac{2\epsilon}{\sigma^3} \frac{\sum_{j \in \tilde{U} \cup L}e^{-\frac{d_{ij}^2}{\sigma^2}
        }d_{ij}^2}{\sum_{j \in \tilde{U} \cup L}e^{-\frac{d_{ij}^2}{\sigma^2}
        }} \\
        &\leq  \epsilon + \frac{2 \epsilon}{\sigma^3} \sum_{j \in \tilde{U} \cup L} e^{-\frac{d_{ij}^2}{2\sigma^2}}d_{ij} &&\text{(Cauchy-Schwartz inequality)} \\
        &\leq  \epsilon + \frac{2 \epsilon}{\sigma^3} \sum_{j \in \tilde{U} \cup L}\sigma e^{-\frac{1}{2}} &&\text{(maximum of $f(x) = xe^{-\frac{x^2}{2c^2}}$ attained at $x = c$)} \\
        &\leq  \epsilon\left(1 + \frac{2 (|L_{\text{Labels}}| + |\tilde{U}_{\text{Labels}}|)}{\sigma^2}\right) \\
        &\leq  \epsilon\left(1 + \frac{2 (|L_{\text{Labels}}| + |\tilde{U}_{\text{Labels}}|)}{\sigma_{\min}^2}\right). 
    \end{align*}
In a similar manner to Theorem \ref{algorithm: harmonic approx}, we need 
$$\epsilon' = \frac{\sqrt{|L_\text{Labels}|}\epsilon^2 \lambda_{\min}(A)}{\lambda}$$ to achieve $\epsilon$ approximations of $\frac{\partial f}{\partial \sigma}$ and $\Tilde{f}$.
Setting $$\epsilon'' = \frac{\epsilon^2 \sigma_{\min}^2\lambda_{\min}(A)}{(2(|L_{\text{Labels}}| + |\tilde{U}_{\text{Labels}}|) + \sigma^2_{\min})\sqrt{|L_\text{Labels}|}\lambda}$$
we also achieve
$$\left|\frac{\partial \tilde{f}_i}{\partial \sigma}_{\epsilon'} - \frac{\partial \tilde{f}_i}{\partial \sigma}\right| < \epsilon.$$
As a result, we obtain the desired bound
$$\left|\tilde{f}_{u}(\sigma)\frac{\partial \tilde{f}_u}{\partial \sigma} - \left(\tilde{f}_{u}(\sigma)\frac{\partial \tilde{f}_u}{\partial \sigma}\right)_\epsilon\right| < \epsilon.$$
Since we have that $A$ is positive definite, via \cite{AXELSSON1976123}, This can be achieved in 
$$O\left(\sqrt{\kappa(A)}\log \frac{1}{\epsilon''}\right)=O\left(\sqrt{\kappa(A)}\log\left(\frac{\lambda (|L_{\text{Labels}}| + |\tilde{U}_{\text{Labels}}|)}{\epsilon \sigma_{\min}\lambda_{\min}(A)}\right)\right)$$ iterations of the CG method.
\end{proof}

\section{Convergence of Nesterov's Gradient Descent and Newton's Method}

In this section we provide useful lemmas that provide convergence analysis for Nesterov's gradient descent and Newton's method, when working with approximate gradients. First we provide a guarantee for Nesterov's method in Theorem \ref{thm:nst}, which uses the result of \cite{d2008smooth} to analyse our algorithm.

\begin{theorem} \label{thm:nst}
    \textit{Suppose the function } $f : \mathbb{R} \rightarrow \mathbb{R}$ \textit{is convex and differentiable, and that its gradient is Lipschitz continuous with constant }$L > 0$\textit{, i.e. we have that }$|f'(x) - f'(y)| \leq L |x - y|$\textit{ for any }$x,y$. \textit{Then if we run Nesterov's method to minimize }$g(\sigma) = (f(\sigma) - \frac{1}{2})^2$ \textit{ on some range $[\sigma_{\text{min}}, \sigma_{\text{max}}]$ where $\left|\frac{\partial f}{\partial \sigma}\right| < \frac{1}{\epsilon \lambda_{\min}(G_A)}$ using $\frac{\partial g}{ \partial \sigma}$ as defined in Algorithm \ref{algorithm: semi harmonic} and finding soft labels and derivatives as defined by some algorithm $A$, we can achieve an $\epsilon$ approximation $\sigma^*_{\epsilon}$ of the optimal result $\sigma^*$ satisfying $|\sigma^*_{\epsilon} - \sigma^*| < \epsilon$ in } $O(\log \log \frac{1}{\epsilon})$ \textit{ iterations of nesterovs method.} We use $O(\text{CG}_A(\frac{\epsilon}{42(\sigma_{\max} - \sigma_{\min})})\log \log \frac{1}{\epsilon})$ conjugate gradient iterations overall, where $\text{CG}_A(\epsilon')$ is the number of conjugate gradient iterations used by algorithm $A$ to achieve $\epsilon'$ approximations of $f, \frac{\partial f}{\partial \sigma}$ satisfying $|f_{u, \epsilon}(\sigma)\frac{\partial f}{\partial \sigma}_\epsilon - f_{u}(\sigma) \frac{\partial f}{\partial \sigma}| < \epsilon'$
\end{theorem}
\begin{proof}
 First, note that 
 \begin{align*}
     \left|\frac{\partial g_u}{\partial \sigma} - \frac{\partial g_u}{\partial \sigma}_{\epsilon'}\right|
 &= \left|2\left(f_u(\sigma) - \frac{1}{2}\right)\left(\frac{\partial f_u}{\partial \sigma}\right) - 2\left(f_u(\sigma)_{\epsilon'} - \frac{1}{2}\right)\left(\frac{\partial f_u}{\partial \sigma}_{\epsilon'}\right)\right|\\
&\leq 4 \epsilon' f_u(\sigma) \frac{\partial f_u(\sigma)}{\partial \sigma} + 2 (\epsilon')^2 f_u(\sigma) \frac{\partial f_u(\sigma)}{\partial \sigma} + \epsilon' \frac{\partial f_u(\sigma)}{\partial \sigma}\\
&\leq 7 \left|f_{u}(\sigma)\frac{\partial f_u}{\partial \sigma} - \left(f_{u}(\sigma)\frac{\partial f_u}{\partial \sigma}\right)_\epsilon\right|.
 \end{align*}

 \noindent Letting $$\epsilon' = \frac{\epsilon}{42 (\sigma_{\text{max}} - \sigma_{\text{min}})}$$
 we find $\epsilon'$ approximations of $f$ and $\frac{\partial f_u}{\partial \sigma}$ in $\text{CG}_A(\epsilon')$ steps. 
 We can then bound $$\left|\left(\frac{\partial g}{\partial \sigma}\right)_{\epsilon'} - \left(\frac{\partial g}{\partial \sigma}\right)\right| \leq \frac{\epsilon}{6 (\sigma_{\text{max}} - \sigma_{\text{min}})}.$$
 On compact set $[\sigma_{\text{min}}, \sigma_{\text{max}}]$ with this bound, we then have that 
 $$\left|\left\langle \left(\frac{\partial g}{\partial \sigma}\right)_{\epsilon'} - \left(\frac{\partial g}{\partial \sigma}\right), y - z \right\rangle\right| \leq \frac{\epsilon}{6} \ \forall y,z \in [\sigma_{\text{min}}, \sigma_{\text{max}}].$$
 With this, \cite{d2008smooth} shows that Nesterov's accelerated gradient descent using an approximate gradient will converge to within $\epsilon$ of the optimal $\sigma^* \in [\sigma_{\text{min}}, \sigma_{\text{max}}]$ in $O(\frac{1}{\sqrt{\epsilon}})$ complexity. This yields $O(\log \log \frac{1}{\epsilon})$ steps until convergence\\\\
 Next we analyze the number of iterations of the CG method used. We called algorithm $A$ $O(\log \log \frac{1}{\epsilon})$ times, each time using $\text{CG}_A(\epsilon') = \text{CG}_A\left(\frac{\epsilon}{42(\sigma_{\max} - \sigma_{\min})}\right)$ iterations. This yields
 $$O\left(\text{CG}_A\left(\frac{\epsilon}{42(\sigma_{\max} - \sigma_{\min})}\right)\log \log \frac{1}{\epsilon}\right)$$
 overall iterations of the CG method to find $\sigma^*$.
\end{proof}

\noindent We also provide an analysis for convergence of Newton's method using approximate gradients in Theorem \ref{thm:ns}.

\begin{theorem}\label{thm:ns}
\textit{Suppose the function $f: \mathbb{R} \rightarrow \mathbb{R}$ has multiplicity 2 at optimal point $x^*$, with $f(x^*) = 0$. If Newton's accelerated method $x_{n + 1} = x_n - 2\frac{f(x_n)}{f'(x_n)}$ coverges quadratically, then so does an epsilon approximation $x_{n + 1} = x_n  - \frac{f(x_n)}{f'(x_n)_\epsilon}$} satisfying $|f'(x)_\epsilon - f'(x)| \leq \epsilon |f(x)| \forall x \in \mathbb{R}$
\end{theorem}

\begin{proof}
First, quadratic convergence of accelerated Newton's method gives us $e_{n + 1} \leq L e_n^2$ for some constant $L$, where $e_n = x^* - x_n$ is the error for the accelerated Newton's method update, and $x_{n + 1} = x_n - 2\frac{f(x_n)}{f'(x_n)}$.

Using the Lagrange form of the Taylor series expansion, we see that
$$f(x_n) = f(x^*) + f'(x^*)(x^* - x_n) + (x^* - x_n)f''(\xi)$$
with $\xi$ between $x^k$ and $x^*$. 
Letting $x^*$ be the optimal point with $f(x^*) = 0, f'(x^*) = 0$ by multiplicity 2, we see that $f(x_k) =  (x^* - x_n)f''(\xi)$. 
Now to handle the $\epsilon$-approximate case note that
\begin{align*}
     e_{n + 1} 
     &= x^* - x_n - 2\frac{f(x^k)}{f'(x_n)_\epsilon} \\
    &\leq  x^* - x_n - 2\frac{f(x^k)}{f'(x_n)(1 + \epsilon)} \\
    &= x^* - x_n - \frac{2f(x^k)}{f'(x_n)} + \frac{2\epsilon}{1 + \epsilon}f(x_n) \\
    &\leq  L e_n^2 + \frac{2\epsilon}{1 + \epsilon}f(x_n) \\
    &\leq  L e_n^2 + \frac{2\epsilon}{1 + \epsilon} (x^* - x_n)^2f''(\xi)\\
    &\leq  \left(L + \frac{2\epsilon}{1+ \epsilon}f''(\xi)\right)e_n^2
\end{align*}
as $x_n \rightarrow x^*$, we see that this is quadratic convergence if we are sufficiently close to $x^*$ ($f''(\xi) < C f''(x^*) \forall \xi \in [x_n, x^*]$). 
\end{proof}

\section{Proof Details from Section \ref{sec:algo}}

\begin{theorem}\label{thm:harmbound}
    Given an algorithm for computing $\epsilon$-approximate soft labels and gradients for the harmonic objective of \citet{zhu2003semi} (\hyperref[summary: harmonic approx]{\textsc{ZGL03Approx}}), if the soft label function $f_u(\sigma)$ is convex and smooth, Algorithm \ref{algorithm: semi harmonic} computes $(\epsilon,\epsilon)$-approximate semi-bandit feedback for the semi-supervised loss $l(\sigma)$ %
    in time $O\left(|E_G|n\sqrt{\kappa(\cL_{UU})}\log\left(\frac{n\Delta}{\epsilon \lambda_{\min}(\cL_{UU})}\right)\log\log\frac{1}{\epsilon}\right)$, where $|E_G|$ is the number of edges in graph $G$, $\cL_{UU}=I - P_{UU}$ is the normalized grounded graph Laplacian (with labeled nodes grounded), $\Delta=\sigma_{\max} - \sigma_{\min}$ is the size of the parameter range and $\kappa(M)=\frac{\lambda_{\max}(M)}{\lambda_{\min}(M)}$ denotes the condition number of matrix $M$.
\end{theorem}
\begin{proof}

As in \citet{balcan2021data}, note that any boundary $\sigma_{\min}$ or $\sigma_{\max}$ must have some $f_{u}(\sigma) = \frac{1}{2}$. Algorithm \ref{algorithm: semi harmonic} finds these boundary pieces by finding roots/zeros of $\left(f_u(\sigma) - \frac{1}{2}\right)^2$. As noted in Theorems C.1 and C.2, both Nesterov's and Newton's descent methods have quadratic convergence, so at every update step in algorithm \ref{algorithm: semi harmonic} (lines 12 and 15), we converge quadratically, leading to $\log \log (\frac{1}{\epsilon})$ update steps needed to satisfy $|\sigma_\epsilon^* - \sigma^*| < \epsilon$, where $\sigma^*$ is the root with $g_u(\sigma^*) = 0$.\\\\
In Theorems \ref{thm:harmapprox} and \ref{thm:delapprox}, we assumed that $|\frac{\partial f}{\partial \sigma}| < \frac{1}{\epsilon \lambda_{\min}(G)}$ for some graph $G$. Consider this is not the case. We examine the Newton update, which is an upper bound on the size of the update step used as our update uses the minimum of Newton and Nesterov steps: 
\begin{align*}
     2 \cdot &\frac{g_u(\sigma)}{g'_u(\sigma)} 
     = 2 \cdot \frac{(f_u(\sigma) - 1/2)^2}{2 \cdot (\partial f/\partial \sigma)(f_u(\sigma) - 1/2)} \\
      &= \frac{(f_u(\sigma) - 1/2)}{(\partial f / \partial \sigma)} \\
     &< \epsilon \lambda_{\min}(G)(f_u(\sigma) - 1/2) &&(\because |{\partial f}/{\partial \sigma}| > {1}/{\epsilon \lambda_{\min}(G)})\\
     &< \epsilon &&(\because\text{$f_u(\sigma)\le1/\lambda_{\min}$, cf. Thms \ref{thm:harmapprox} and \ref{thm:delapprox}}). \\
\end{align*}
Thus in this case the update step is less than $\epsilon$, and we will terminate after one subsequent step.

 As noted in Theorem \ref{thm:harmapprox}, we need $O\left(\sqrt{\kappa(\cL_{UU})}\log\left(\frac{n}{\epsilon' \lambda_{\min}(\cL_{UU})}\right)\right)$ CG steps to reach an $\epsilon'$ approximation of $f \frac{\partial f}{\partial \sigma}$. Theorem \ref{thm:nst} states that we need $\epsilon' = O\left(\frac{\epsilon}{\Delta}\right)$ to find an $\epsilon$ approximation of the root $\sigma^*$, so this takes complexity 

$$O\left(\sqrt{\kappa(\cL_{UU})}\log\left(\frac{n\Delta}{\epsilon \lambda_{\min}(\cL_{UU})}\right)\right).$$

Running a single iteration of the conjugate gradient method requires a constant number of matrix-vector products of form $Ax$, where $A$ is the weighted adjacency matrix for graph $G$. This computation takes $O(|E_G|)$ time. %
Finally, we run this algorithm for each of the $n$ points, leading to an overall time complexity of 

$$O\left(|E_G|n\sqrt{\kappa(\cL_{UU})} \log\left(\frac{n\Delta}{\epsilon \lambda_{\min}(\cL_{UU})}\right)\log\log\frac{1}{\epsilon}\right).$$

If $G$ is the complete graph, $|E_G| \in O(n^2)$. If $G$ is a kNN graph for some fixed $k$, then $|E_G| = kn \in O(n)$.  
\end{proof}

\section{Experiment Details and Insights}

We include further experimental details below, including implementation and insights into further challenges as well as potential future work.

\begin{figure}
\centering
\subfloat[Algorithm \ref{algorithm: harmonic approx} \\ $|U| = 100$]{\includegraphics[width = 2in]{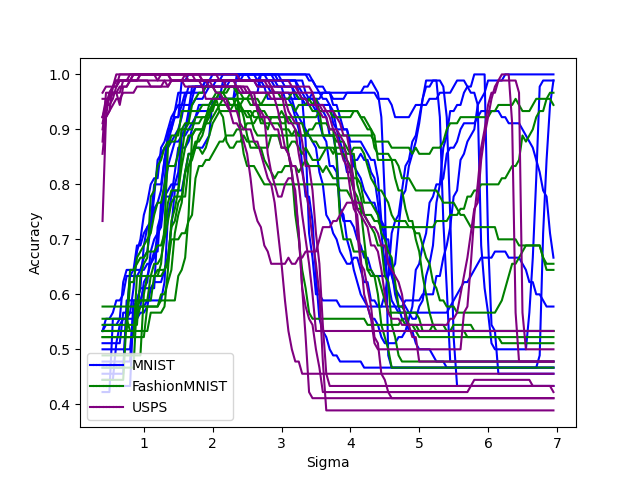}}
\subfloat[Algorithm \ref{algorithm: harmonic approx} (kNN) \\ $|U| = 100$]{\includegraphics[width = 2in]{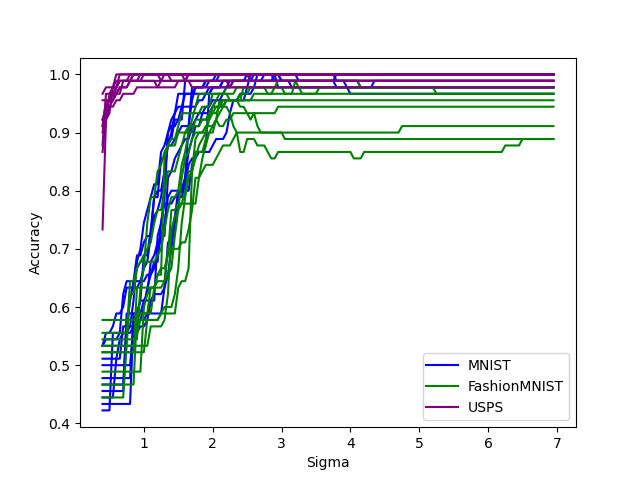}}
\subfloat[Algorithm \ref{algorithm: delalleau approx} (kNN) \\ $|\Tilde{U}| = 100, |U| = 1000$]{\includegraphics[width = 2in]{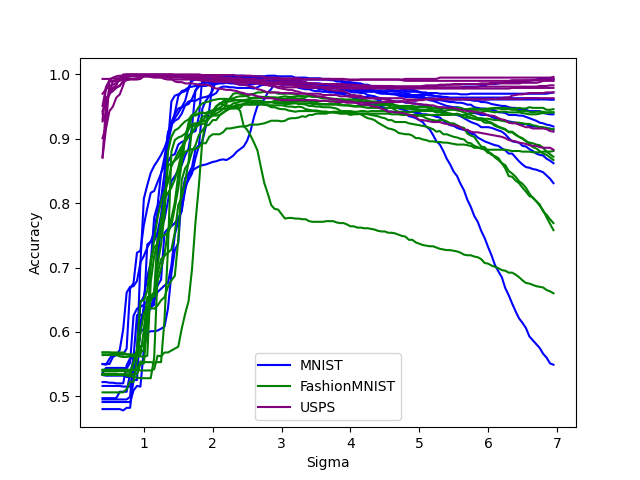}}
\caption{Accuracy values across $\sigma$ for different approaches using the CG Method with 20 iterations.}
\label{fig:intervals}
\end{figure}

\begin{figure}
\centering
\subfloat[MNIST]{\includegraphics[width = 2in]{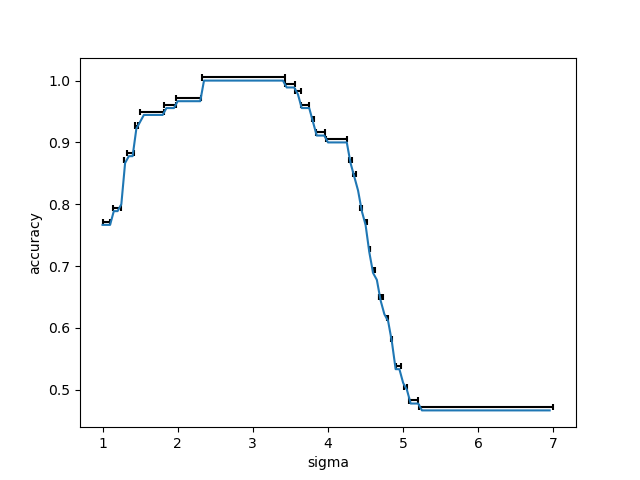}}
\subfloat[Fashion-MNIST]{\includegraphics[width = 2in]{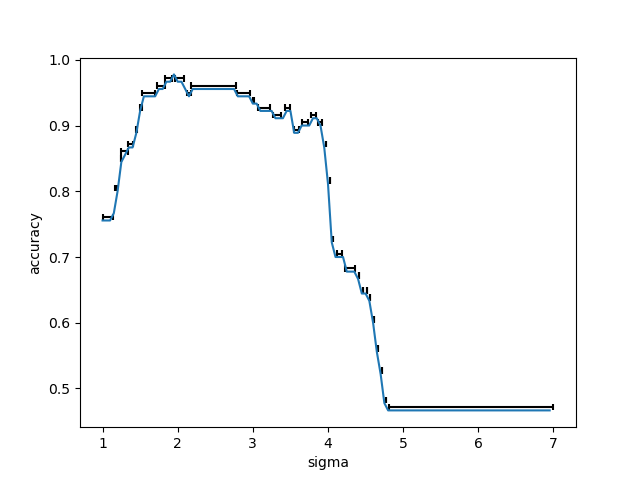}}
\subfloat[USPS]{\includegraphics[width = 2in]{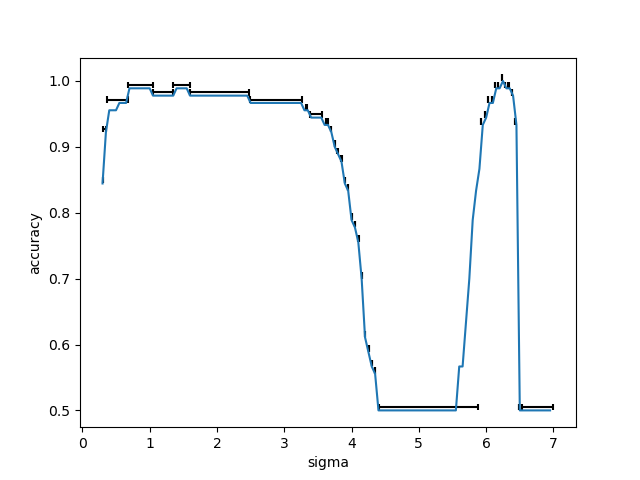}}
\caption{Interval calculation with labeling via Algorithm \ref{algorithm: harmonic approx} (complete graph) $|U| = 100$.}
\label{fig:naive}
\end{figure}

\begin{figure}
\centering
\subfloat[MNIST]{\includegraphics[width = 2in]{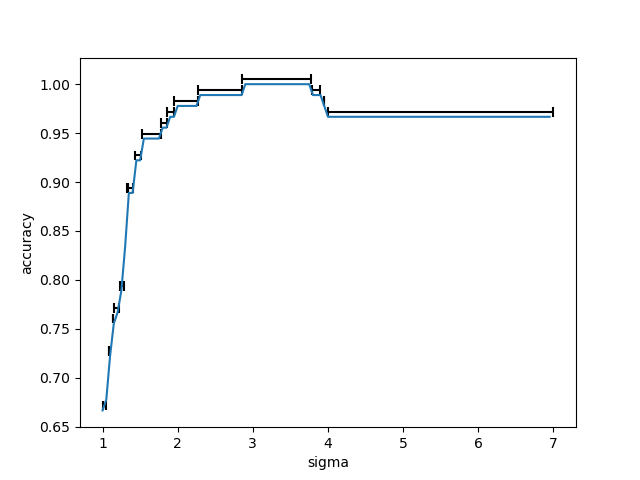}}
\subfloat[Fashion-MNIST]{\includegraphics[width = 2in]{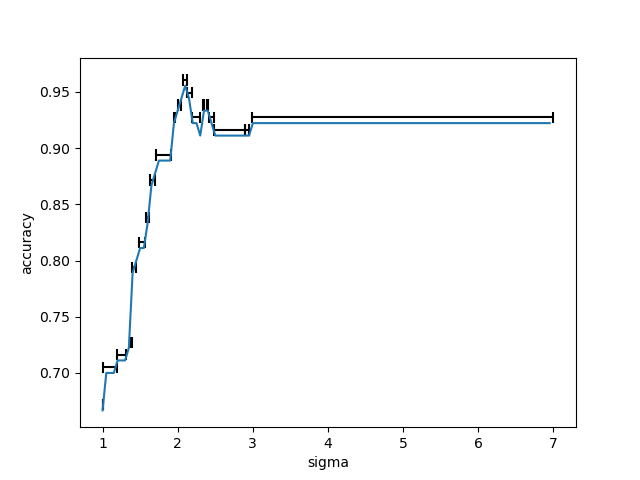}}
\subfloat[USPS]{\includegraphics[width = 2in]{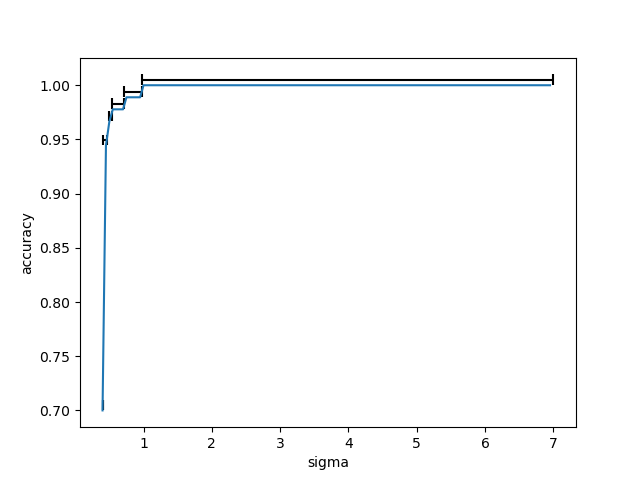}}
\caption{Interval calculation with labeling via Algorithm \ref{algorithm: harmonic approx}, kNN with $k=6$, $|U| = 100$.}
\label{fig:kNN}
\end{figure}

\begin{figure}[b]
\subfloat[$\sqrt{\kappa(I - P_{UU})}$ (\ref{thm:harmboundmain})]{\includegraphics[width = 1.6in, height=1.32in]{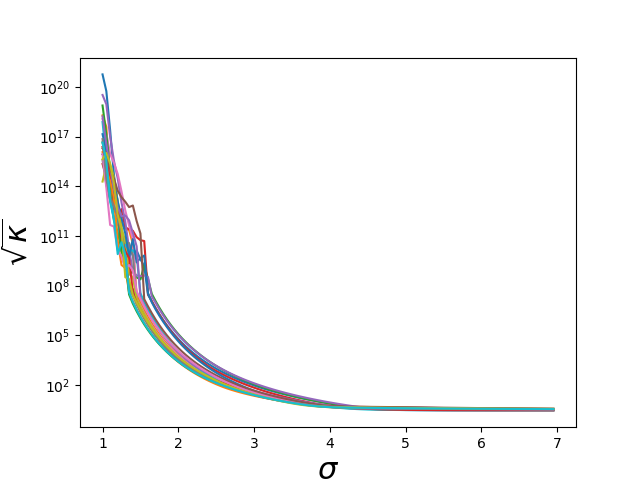}} 
\subfloat[$\sqrt{\kappa(I - P_{UU})}$ (\ref{thm:harmboundmain}, kNN)]{\includegraphics[width = 1.6in, height=1.32in]{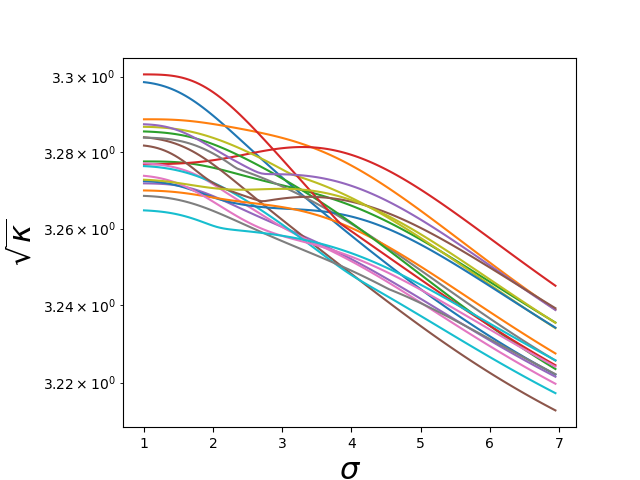}} 
\subfloat[$\sqrt{\kappa(A)}$ (\ref{thm:delbound})]{\includegraphics[width = 1.6in]{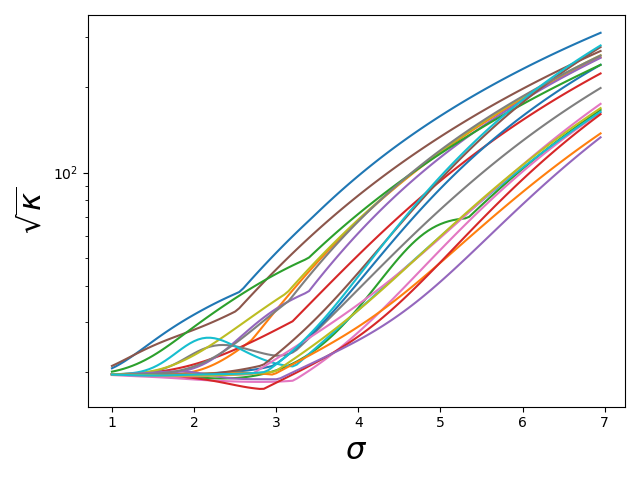}} 
\subfloat[$\sqrt{\kappa(A)}$ (\ref{thm:delbound}, kNN)]{\includegraphics[width = 1.6in]{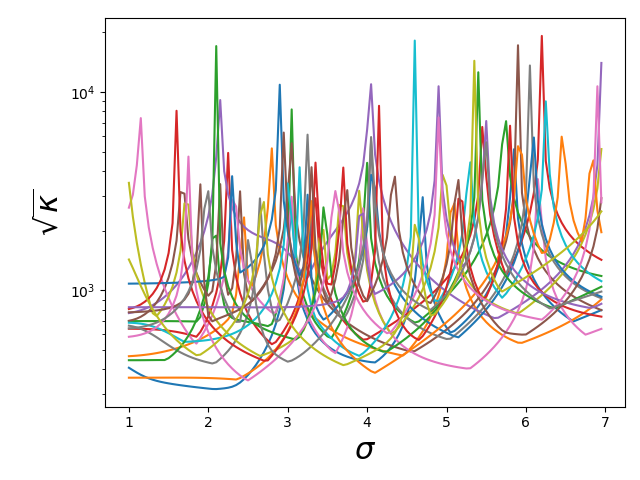}} 
\caption{Condition numbers for matrices from Theorems \ref{thm:harmboundmain} and \ref{thm:delbound} for MNIST subsets size 100.}
\label{fig:kappa}
\end{figure}

\subsection{Implementation Details} 
For all experiments, we consider 10 random subsets of datasets MNIST, FashionMNIST, and USPS. We will consider a bounded parameter domain to avoid highly ill-conditioned graph matrices (Figure \ref{fig:kappa}). We pick $\sigma_{\text{min}}$ based on behavior of graph condition numbers for low $\sigma$ values, where $\sigma_{\text{min}}$ is 1 for MNIST and FashionMNIST using the CG method, and .4 for USPS using the CG method. We keep $\sigma_{\text{max}} = 7$ for all experiments. We find that Algorithm \ref{algorithm: semi harmonic} does not produce valid intervals when condition number is high, and note that ill-conditioned graphs lead to low accuracy. In Figure \ref{fig:intervals}, we see that USPS has higher accuracy values in range $[.4,1]$ while MNIST and FashionMNIST do not display optima until later $\sigma$ values. We find that computing a full matrix inverse is less stable than the CG method for low $\sigma$ values, and use $\sigma_{\text{min}} = 2$ for full inverse interval calculation. We keep $\sigma_{\text{min}} = 1$ always when calculating average number of intervals overall in Table \ref{table:regularIntervals} %
in order to compare number of intervals on the same range ([1,7]) for all problem instances. 

   \noindent In order to find intervals, we begin with $\sigma_0 = \sigma_{\text{min}}$. Once interval $[\sigma_{l}, \sigma_{u}]$ is calculated for $\sigma_0 = \sigma_{\text{min}}$, we let $\sigma_0^{(1)} = \sigma_{u} + \text{step}$ as the next initial $\sigma$. Here we use $\text{step} = .05, \epsilon = 1e-4, \eta = 1$, where $\epsilon$ and $\eta$ are used as in Algorithm \ref{algorithm: semi harmonic}. We also consider algorithmic optimizations to speedup runtime and improve performance of Algorithm \ref{algorithm: semi harmonic}, which can be found in Appendix \ref{sec:algopt}.  

\begin{figure}[b]
\centering
\subfloat[MNIST]{\includegraphics[width = 2in]{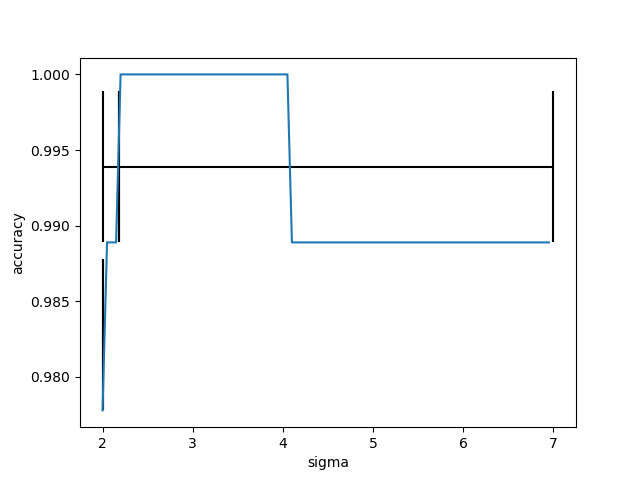}}
\subfloat[Fashion-MNIST]{\includegraphics[width = 2in]{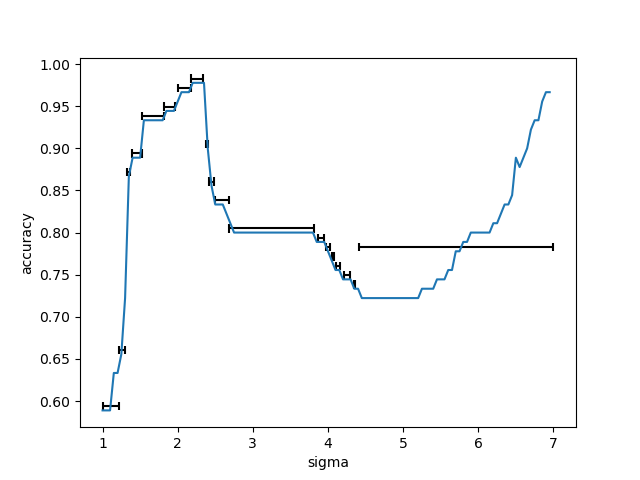}}
\subfloat[USPS]{\includegraphics[width = 2in]{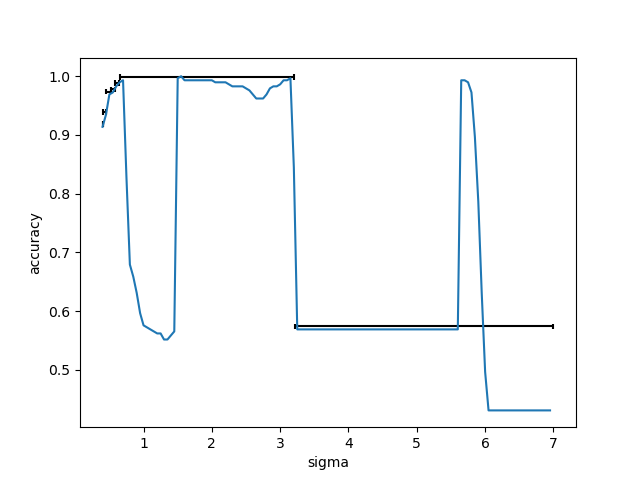}}
\caption{Challenging cases for algorithm}
\label{fig:failure}
\end{figure}

\subsection{Algorithm Optimization}\label{sec:algopt}
\noindent A few optimizations of Algorithm \ref{algorithm: semi harmonic} were used in practice. First, note that if we update the left endpoint $\sigma_l$ of the piecewise constant interval containing $\sigma_0$ on line \ref{algSemiHarmonic-lineupdateint}, then we need not consider any root $\sigma_{l'}$ with $\sigma_{l'} < \sigma_l$, as it will not change our current left endpoint of the interval. As a result, we can stop the while loop on line \ref{algSemiHarmonic-linewhile} of Algorithm \ref{algorithm: semi harmonic} if we leave current interval range $[\sigma_l, \sigma_h]$ at any point in the algorithm. Second, we change the while loop on line \ref{algSemiHarmonic-linewhile} to:
$$|\sigma_{n + 1} - \sigma_n| \ge \epsilon \text{ OR } |f_u(\sigma_u)_{n + 1} - f_u(\sigma_u)_n| \ge \epsilon $$
Noting that $f$ values can go through very short periods of high change as evidenced by Figure \ref{fig: gd+ns}, if $f$ values have some large change in a given step, then we may be making progress towards a critical point even if $\sigma$ has not changed drastically. Further, if both quantities are under $\epsilon$ yet we have a label that is not close to .5, we stop the algorithm prematurely without having found a critical point.

\noindent Finally, for $k$-nearest neighbor graphs, we note that the Gaussian kernel preserves order across $\sigma$, i.e. 
$$d(a,b) < d(c,d) \implies e^{-\frac{d(a,b)^2}{\sigma^2}} > e^{-\frac{d(c,d)^2}{\sigma^2}} \quad\forall\; \sigma \in (0, \infty).$$

\begin{table}[t]
\centering
\begin{tabular}{ |c|c|c|c| } 
\hline
\multirow{2}{4em}{Dataset} & \multirow{2}{2em}{Size} & 
\multirow{2}{9em}{Time per Inverse (s), Full Inverse} & 
\multirow{2}{8em}{Optimal Accuracy, Full  Inverse} \\
&&& \\\hline
\multirow{3}{4em}{MNIST}
& 500 & 0.1285 & 0.9988 \\
& 1000 & 0.2248 & 0.9991 \\
& 2000 & 0.5528 & 0.9986  \\ 
\hline
\multirow{3}{4em}{Fashion-MNIST}
& 500 & 0.1275 & 0.9502  \\
& 1000 & 0.2312 & 0.9775 \\
& 2000 & 0.5454 & 0.9570   \\
\hline
\multirow{3}{4em}{USPS}
& 500 & 0.1445 & 0.9998   \\
& 1000 & 0.2230 & 0.9997  \\
& 2000 & 0.5437 & 1.0   \\
\hline
\end{tabular}   
\caption{Optimal Accuracy/Average Time computing matrix inverse via Algorithm \ref{algorithm: harmonic approx} (Averaged over 10 samples).}
\label{table:invtimefull}
\end{table}

\begin{table}[t]
\centering
\begin{tabular}{ |c|c|c|c| } 
\hline
\multirow{2}{4em}{Dataset} & \multirow{2}{2em}{Size} & 
\multirow{2}{9em}{Time per Inverse (s), Full Inverse} & 
\multirow{2}{8em}{Optimal Accuracy, Full  Inverse} \\
&&& \\
\hline
\multirow{3}{4em}{MNIST}
& 500 & 0.1235 & 0.999\\
& 1000 & 0.2181 & 0.9993\\
& 2000 & 0.5354 & 0.9992\\
\hline
\multirow{3}{4em}{Fashion-MNIST}
& 500 & 0.1299 & 0.9637\\
& 1000 & 0.2244 & 0.9638\\
& 2000 & 0.5337 & 0.9683\\
\hline
\multirow{3}{4em}{USPS}
& 500 & 0.1254 & 0.9998\\
& 1000 & 0.2189 & 0.9998\\
& 2000 & 0.5411 & 1.0\\
\hline
\end{tabular}   
\caption{Optimal Accuracy/Average Time computing matrix inverse \textbf{with kNN=6} via Algorithm \ref{algorithm: harmonic approx} (Averaged over 10 samples).}
\label{table:invtimefull}
\end{table}

As a result, we only need to calculate $k$-nearest neighbors once when finding an interval centered around $\sigma_0$. After a kNN mask is computed for $W_{\sigma}$, it can then be used for any subsequent $W_{\sigma'}$. When analyzing the time to compute all intervals in a given range, this could be computed once for all starting points, but since we were interested in time for each interval, we calculated the mask for every interval.

\subsection{Challenging Cases}

There were certain challenging problem instances associated with Algorithm \ref{algorithm: semi harmonic}. First, we considered using gradient descent methods with adaptive step size \citep{VRAHATIS2000367} to combat the issue of very different gradient values at different $\sigma$ values, but we found this method to be ineffective for our specific task. Some of the datasets would return singular matrices or non-convergent matrices for very low $\sigma$ values, leading to interval calculation needing to be started at some later point, as mentioned in Section \ref{sec:FeedbackSet}. Similarly, we find that the algorithm was more likely to miss intervals for very high values of $\sigma$ ($\geq 6$) as seen in Figure \ref{fig:failure}. This could be due to very different or lower condition numbers as compared to earlier $\sigma$ values as evidenced by Figure \ref{fig:kappa}. One solution could be updating the learning rate $\eta$ as a function of condition number or $\sigma$ value. We also find an outlier graph in the kNN Delalleau graph family displayed in Figure \ref{fig:kappa} that behaved similarly to the harmonic minimizer graphs for low $\sigma$. In addition, we find that the algorithm may not be able to correctly find the rightmost point of very long piecewise intervals (size 3 or more), as the descent algorithms has trouble finding critical points that are very far from the initial $\sigma$.

\begin{table}[t]
\centering
\begin{tabular}{ |c|c|c|c|c|c|c|c| } 
\hline
\multirow{2}{4em}{Dataset} & \multirow{2}{2em}{Size} & 
\multirow{2}{5em}{Time,\\ CG, $t=5$ } & 
\multirow{2}{5em}{Time,\\ CG, $t=10$} & 
\multirow{2}{5em}{Time,\\ CG, $t=20$} & 
\multirow{2}{5em}{Accuracy, CG, $t=5$} &
\multirow{2}{5em}{Accuracy, CG, $t=10$} &
\multirow{2}{5em}{Accuracy, CG, $t=20$} \\
&&&&&&& \\
\hline
\multirow{3}{4em}{MNIST}
& 500 & 0.004 & 0.0041 & 0.004 & 0.9971 & 0.9971 & 0.9971\\
& 1000 & 0.0058 & 0.0058 & 0.006 & 0.9958 & 0.9958 & 0.9958\\
& 2000 & 0.0238 & 0.0234 & 0.0234 & 0.9956 & 0.9956 & 0.9956\\
\hline
\multirow{3}{4em}{Fashion-MNIST}
& 500 & 0.0041 & 0.004 & 0.004 & 0.9561 & 0.9561 & 0.9561\\
& 1000 & 0.0058 & 0.0059 & 0.0059 & 0.9544 & 0.9544 & 0.9544\\
& 2000 & 0.0235 & 0.0256 & 0.0241 & 0.9579 & 0.9579 & 0.9579\\
\hline
\multirow{3}{4em}{USPS}
& 500 & 0.0041 & 0.0041 & 0.004 & 0.9945 & 0.9945 & 0.9945\\
& 1000 & 0.0058 & 0.0056 & 0.006 & 0.9989 & 0.9989 & 0.9989\\
& 2000 & 0.0234 & 0.0234 & 0.0235 & 0.9792 & 0.9792 & 0.9792\\
\hline
\end{tabular}   
\caption{Optimal Accuracy/Average Time computing approximate matrix inverse via Algorithm \ref{algorithm: harmonic approx} (Averaged over 10 samples).}
\label{table:CGtimefull}
\end{table}

\begin{table}[h]
\centering
\begin{tabular}{ |c|c|c|c|c|c|c|c| } 
\hline
\multirow{2}{4em}{Dataset} & \multirow{2}{2em}{Size} & 
\multirow{2}{5em}{Time,\\ CG, $t=5$ } & 
\multirow{2}{5em}{Time,\\ CG, $t=10$} & 
\multirow{2}{5em}{Time,\\ CG, $t=20$} & 
\multirow{2}{5em}{Accuracy, CG, $t=5$} &
\multirow{2}{5em}{Accuracy, CG, $t=10$} &
\multirow{2}{5em}{Accuracy, CG, $t=20$} \\
&&&&&&& \\
\hline
\multirow{3}{4em}{MNIST}
& 500 & 0.0036 & 0.0034 & 0.0034 & 0.9988 & 0.9988 & 0.9988\\
& 1000 & 0.0052 & 0.005 & 0.0049 & 0.9865 & 0.9865 & 0.9865\\
& 2000 & 0.0224 & 0.0225 & 0.0225 & 0.9754 & 0.9754 & 0.9754\\
\hline
\multirow{3}{4em}{Fashion-MNIST}
& 500 & 0.0033 & 0.0034 & 0.0034 & 0.9692 & 0.9692 & 0.9692\\
& 1000 & 0.0053 & 0.0053 & 0.0052 & 0.9714 & 0.9714 & 0.9714\\
& 2000 & 0.0224 & 0.0224 & 0.0225 & 0.9723 & 0.9723 & 0.9723\\
\hline
\multirow{3}{4em}{USPS}
& 500 & 0.0033 & 0.0034 & 0.0035 & 1.0 & 1.0 & 1.0\\
& 1000 & 0.0048 & 0.0049 & 0.0048 & 1.0 & 1.0 & 1.0\\
& 2000 & 0.0225 & 0.0225 & 0.0226 & 0.9943 & 0.9943 & 0.9943\\
\hline
\end{tabular}   
\caption{Optimal Accuracy/Average Time computing approximate matrix inverse via Algorithm \ref{algorithm: harmonic approx} \textbf{with kNN=6} (Averaged over 10 samples)}
\label{table:CGtimekNN}
\end{table}

\subsection{Further Directions}
While we used our method on two algorithms, namely the SSL approaches proposed by \cite{delalleau2005efficient} and \cite{zhu2003semi}, our method could be extended to other SSL labeling schemes that allow for a derivative $\frac{\partial f}{\partial \sigma}$ to be taken, where $f$ is the labeling function and $\sigma$ is a hyperparameter. \\
One technique not explored in this work is anchor graph regularization \citep{liu2010large}. In this method, a set of points is chosen to be "anchor points". These points are then labeled, and all other points are labeled via some weighted combination of the labels of the anchor points. Since these anchor points are chosen via a k-means clustering idea, in order to calculate $\frac{\partial f}{\partial \sigma}$, it is necessary to determine how the cluster centers, or anchor points, of a graph change as the parameter $\sigma$ changes. We leave this as a further direction. Another SSL technique that can be explored as a further direction is the use of leading eigenvectors as "anchor points" \citep{sinha2009semi}. In this method, a lasso least squares approximation is used with respect to certain eigenvectors of the kernel matrix to create a point classifier. While the lasso least squares method has a closed form solution, it is necessary to determine how these eigenvectors change as a function of $\sigma$. It is possible that this change could be approximated fast using eigenvector approximation techniques like Lanczos algorithm, implementation and verification is an interesting candidate for further research. 

While we use the time saved to run Algorithm \ref{algorithm: harmonic approx} faster by computing inverses from \cite{delalleau2005efficient} and \cite{zhu2003semi} quickly, it could also be used directly with results from any graph based SSL technique that employs radial basis kernels and takes inverses. Further, since it is standard to use kNN graphs when doing matrix inverses as in \cite{delalleau2005efficient} and \cite{zhu2003semi}, this would speedup an inverse on an $m \times m$ matrix from $O(m^3)$ to $O(m)$ for an $\epsilon$ approximation. As a result, larger subsets of data could be inverted. It has been shown that accuracy increases as more data is used for inversion \citep{delalleau2005efficient, zhu2003semi}, and we have shown that we can match optimal accuracy of full matrix inversion across hyperparameters with the CG-method. This signifies that using a larger dataset and finding approximate inverses with the CG method could lead to higher accuracy than using a smaller dataset and taking exact inverses.

We notice that larger datasets lead to much smaller interval sizes when using Algorithm \ref{algorithm: delalleau approx}. One way to combat this could be to find $(\epsilon, \epsilon)$-approximate intervals whose accuracy values all fall within some $\delta$ of each other, as opposed to being piecewise constant. In this way, we could find larger accuracy-approximate intervals, which then allow for a faster search of range $[\sigma_{\min}, \sigma_{\max}]$. 

\end{document}